\documentclass[10pt]{article} 
\usepackage[accepted]{rlj} 

\usepackage{amssymb}            
\usepackage{mathtools}          
\usepackage{mathrsfs}           
\usepackage{graphicx}           
\usepackage{subcaption}         
\usepackage[space]{grffile}     
\usepackage{url}                
\usepackage{lipsum}             

\usepackage{amsthm}
\usepackage{comment}
\usepackage{booktabs}
 \usepackage{multirow}

 \usepackage{algorithm}
\usepackage{algpseudocode}

\title{Multi-Modal, Multi-Environment Machine Teaching for Robust Reward Learning}
\setrunningtitle{Multi-Modal, Multi-Environment Machine Teaching for Robust Reward Learning}

\author{Ali Larian\textsuperscript{1}, Qian Lin\textsuperscript{1}, Chang Zong Wu\textsuperscript{1}, Daniel S. Brown\textsuperscript{1}}

\emails{\{ali.larian, qian.lin, u1440478, daniel.s.brown\}@utah.edu}

\affiliations{
$^{1}$\textbf{Kahlert School of Computing, University of Utah.}\\
}

\contribution{
Analysis and insights into different feedback types for machine teaching of reward functions in both unlimited-data and limited-budget regimes, showing that comparisons impose the strongest global constraints while demonstrations are more constraint-efficient per query under tight budgets.
}
{
Prior work on teaching Inverse RL agents \citep{buning2022interactive,BrownNiekum2019MachineTeachingIRL} has primarily focused on demonstrations and did not systematically analyze how different feedback types constrain reward recovery.
}

\contribution{
A formal characterization of environment-dependent reward identifiability, showing that even unlimited feedback in a single MDP leaves residual reward ambiguity.
}
{
Existing machine teaching approaches for IRL \citep{buning2022interactive,BrownNiekum2019MachineTeachingIRL} typically assume a fixed environment and do not analyze how environment dynamics affect reward identifiability.
}

\contribution{
Hierarchical Set Cover Optimal Teaching (HSCOT), a framework that selects informative environments and feedback queries to efficiently constrain rewards across multiple MDPs.
}
{
Prior teaching methods \citep{buning2022interactive,BrownNiekum2019MachineTeachingIRL}  operate within a single environment and cannot exploit variation in environment dynamics to reveal complementary reward constraints.
}

\contribution{
Empirical validation showing that HSCOT achieves higher constraint coverage and lower regret on held-out environments compared to uniform teaching under identical budgets.
}
{
Existing evaluations of machine teaching for IRL \citep{buning2022interactive,BrownNiekum2019MachineTeachingIRL} only consider single-environment teaching settings.
}


\summary{
As autonomous agents are increasingly deployed across diverse operational contexts, aligning their behavior with human intent demands reward functions that remain robust to such changes rather than overfitting to any single environment.
Inverse reinforcement learning (IRL) provides a principled way to infer such objectives from human feedback. 
However, existing analyses of optimal teaching approaches for IRL focus on single-environment, demonstration-only settings, leaving underexplored how heterogeneous feedback modalities and environment dynamics jointly constrain reward functions that generalize across multiple environments.
Because demonstrations in one MDP entangle reward information with that environment’s specific structure, the resulting rewards frequently fail to generalize when the agent is deployed in a new setting. 
We first analyze how different feedback modalities constrain rewards, showing that, in the unlimited-data regime, comparisons impose strictly stronger global constraints than other modalities. 
Beyond this theoretical analysis, we introduce a hierarchical machine teaching algorithm for reward learning that operates across multiple MDPs. 
The algorithm first greedily selects informative environments that expose complementary reward constraints, then strategically queries low-cost feedback within those environments. 
Empirically, our method achieves substantially lower regret and stronger generalization to held-out environments than uniform teaching baselines under identical feedback budgets, demonstrating the importance of multi-environment, multi-modal teaching for learning dynamics-robust reward functions.
}

\begin{document}

\theoremstyle{plain}
\newtheorem{theorem}{Theorem}
\newtheorem{lemma}{Lemma}
\newtheorem{proposition}{Proposition}
\newtheorem{corollary}{Corollary}

\theoremstyle{definition}
\newtheorem{definition}{Definition}
\newtheorem{assumption}{Assumption}

\theoremstyle{remark}
\newtheorem{remark}{Remark}

\makeCover  
\maketitle  

\begin{abstract}
As autonomous agents are increasingly deployed across diverse operational contexts, aligning their behavior with human intent demands reward functions that remain robust to such changes rather than overfitting to any single environment.
Inverse reinforcement learning (IRL) provides a principled way to infer such objectives from human feedback. 
However, existing analyses of optimal teaching approaches for IRL focus on single-environment, demonstration-only settings, leaving underexplored how heterogeneous feedback modalities and environment dynamics jointly constrain reward functions that generalize across multiple environments.
Because demonstrations in one MDP entangle reward information with that environment’s specific structure, the resulting rewards frequently fail to generalize when the agent is deployed in a new setting. 
We first analyze how different feedback modalities constrain rewards, showing that, in the unlimited-data regime, comparisons impose strictly stronger global constraints than other modalities. 
Beyond this theoretical analysis, we introduce a hierarchical machine teaching algorithm for reward learning that operates across multiple MDPs. 
The algorithm first greedily selects informative environments that expose complementary reward constraints, then strategically queries low-cost feedback within those environments. 
Empirically, our method achieves substantially lower regret and stronger generalization to held-out environments than uniform teaching baselines under identical feedback budgets, demonstrating the importance of multi-environment, multi-modal teaching for learning dynamics-robust reward functions.
\end{abstract}
\vspace{-15pt}
\section{Introduction}
\label{sec:submission}
\vspace{-5pt}
As autonomous agents become increasingly common, a central challenge is enabling humans to convey their intent in sequential decision-making settings. Rather than merely imitating behavior, agents must infer underlying objectives that remain aligned across varied situations~\citep{kaufmann2024survey}.
In realistic settings, humans provide heterogeneous feedback rather than demonstrations alone \citep{mehta2024unifiedlearning}.
Moreover, agents are rarely confined to a single operating condition; instead, they are deployed across multiple operational contexts that differ in dynamics, constraints, or embodiment. 
A service robot, for instance, may be required to perform the same task across different physical layouts, action limitations, or safety conditions, while still adhering to a consistent human-defined objective — for example, reaching a target location efficiently while avoiding hazards. 
Understanding how to best teach an agent to act appropriately across heterogeneous settings is challenging. 
The teacher must convey an objective that is invariant to changes in dynamics or environment-specific details, while the agent must learn a representation that generalizes beyond the conditions under which feedback is provided. This raises the central question of this work: \emph{How can a human use heterogeneous feedback to teach an agent so that its behavior remains aligned with the human’s intent across multiple operating conditions?}

This question naturally motivates a machine teaching perspective on sequential decision-making, where a teacher strategically selects feedback to induce a desired learning outcome with minimal cost. A natural formalism for representing human intent is reward learning, in which an agent infers a reward function from human-provided feedback \citep{abbeel2004apprenticeship}. The teaching problem can therefore be cast as machine teaching for reward function learning~\citep{BrownNiekum2019MachineTeachingIRL}, where the teacher selects a minimal yet maximally informative set of feedback instances to recover a reward that induces low-regret behavior. 
However, existing machine teaching frameworks for reward learning largely rely on demonstrations as the sole feedback modality \citep{kamalaruban2019interactive,buning2022interactive}. 
Besides the single-modality setting, prior work on machine teaching for IRL has largely restricted teaching to a single environment with fixed dynamics \citep{haug2018teaching,kamalaruban2019interactive,yengera2021curriculum,buning2022interactive}. 
Rewards learned in a single MDP often overfit to environment-specific dynamics, leading to failures when deployed under different dynamics or learning assumptions \citep{booth2023perils,fu2017learning,he2021assisted}. 
We formally characterize this limitation by showing that reward identifiability is environment-dependent. 
Consequently, resolving reward ambiguity often requires teaching across multiple environments rather than collecting more feedback within one.

To address the limitations of single-modality and single-environment teaching, we introduce \textit{Hierarchical Set Cover Optimal Teaching (HSCOT)}, a framework that jointly selects informative environments and heterogeneous feedback queries to constrain the reward space across multiple MDPs. 
Unlike prior approaches that operate within a single environment, HSCOT exploits variation in environment dynamics to expose complementary reward constraints, yielding a hierarchical strategy in which environments are selected first and feedback instances are chosen within them. 
Our contributions are as follows:

\begin{itemize}[leftmargin=12pt, itemsep=1pt, topsep=2pt, parsep=0pt]

\item \textbf{Analysis of feedback modalities for teaching reward learning.}
We examine how different feedback types constrain reward identification in machine teaching for reward learning.
In the unlimited-data regime, comparisons impose the strongest global constraints; however, under limited budgets,  demonstrations are more constraint-efficient per query.

\item \textbf{Environment-dependent reward identifiability.}
We study how reward identifiability depends on environment dynamics and formally 
characterize how additional environments can expose complementary constraint 
directions for better reward identification.

\item \textbf{Hierarchical multi-environment machine teaching.}
We introduce Hierarchical Set Cover Optimal Teaching (HSCOT), a framework that jointly selects informative environments and feedback queries to efficiently constrain rewards across multiple MDPs.

\item \textbf{Empirical validation.}
Experiments show that HSCOT achieves higher constraint coverage and significantly lower regret on held-out environments compared to uniform teaching.
\end{itemize}
Our implementation and experiment code are available online.\footnote{\url{https://github.com/Alilarian/multienv-reward-teaching}}

\vspace{-5pt}
\section{Related work}
\label{sec:others}
\vspace{-5pt}

The problem of identifying minimal teaching sets has been extensively studied
in Algorithmic Teaching \citep{Balbach2009} and Machine Teaching
\citep{liu2017iterative,zhu2018overview,10.5555/2888116.2888288,devidze2020understanding,wang2021teaching}.
These works primarily consider supervised learning settings such as
classification and regression, where teaching reduces to selecting informative
examples for a learner with a fixed hypothesis class.
In contrast, teaching in sequential decision-making requires reasoning about
policies, trajectories, and reward functions, introducing additional structure
absent from standard supervised settings \citep{10.5555/2900929.2900946}.
Our work builds on this paradigm and studies machine teaching for reward learning across multiple environments with heterogeneous
feedback.
The most closely related work studies machine teaching for Inverse Reinforcement Learning (IRL),
introduced by \citet{BrownNiekum2019MachineTeachingIRL}, which formulates
teaching as selecting demonstrations that induce a target reward in a learner
within a single Markov decision process (MDP).
Subsequent work extends this framework to interactive settings where the
teacher adaptively selects demonstrations based on the learner’s current
policy or posterior over rewards
\citep{kamalaruban2019interactive,buning2022interactive}.
Related work considers learner-specific preferences or internal constraints
\citep{tschiatschek2019learner} or feature mismatch
\citep{haug2018teaching}.
However, these approaches assume a single environment and rely on
demonstrations as the primary feedback modality. \citet{pmlr-v139-brown21a} studied the complementary problem of machine testing: how to check whether a learned reward is aligned across multiple environments. 
To our knowledge, we are the first to study machine teaching for reward learning across multiple environments with heterogeneous feedback modalities.

Outside classical IRL teaching, several works study teaching or guidance for
sequential decision-making agents under alternative assumptions, including
curriculum learning \citep{yengera2021curriculum},
planning-based teaching MDPs \citep{peltola2019machine}, and
teaching-by-reinforcement for online RL agents \citep{zhang2021sample}.
However, these methods do not address reward inference or reward generalization across
multiple environments.

A parallel line of work studies reward learning from human feedback,
including demonstrations, rankings, preferences, and corrective signals
\citep{wirth2017survey,christiano2017deep,brown2019extrapolating,
pmlr-v164-myers22a,losey2022physical,wang2025effects}.
Several approaches develop unified probabilistic frameworks for combining
heterogeneous feedback signals
\citep{jeon2020rewardrational,mehta2024unifiedlearning,metz2025reward},
while others study the effects of human rationality assumptions or feedback
noise \citep{ghosal2023effect}.
These works primarily adopt a learner-centric perspective, focusing on how
to infer rewards from feedback rather than how a teacher should select
informative environments and queries.
Recent studies also show that rewards learned in a single environment can
become entangled with environment dynamics and fail to generalize across
settings \citep{fu2017learning,he2021assisted,booth2023perils}. 
An orthogonal line of work, meta-IRL \citep{yu2019meta,chen2024meta}, learns priors for fast task-specific reward learning, whereas we aim to teach a single reward that generalizes across environments without further fine-tuning. 
While prior machine teaching approaches implicitly assume that sufficient
demonstrations in one MDP can identify the reward, some reward ambiguities are structural rather than data-driven.
Our work addresses this limitation by teaching across multiple MDPs,
where variation in dynamics reveals complementary reward constraints.


Conceptually, our approach extends the Behavioral Equivalence Class
framework of \citet{BrownNiekum2019MachineTeachingIRL} beyond
single-MDP demonstration teaching.
We define \emph{generalized behavioral equivalence classes},
which unify reward constraints induced by heterogeneous feedback
across multiple environments.
We further propose a hierarchical teaching structure that separates
environment and feedback selection: the outer level selects
environments whose dynamics expose complementary reward constraints,
while the inner level chooses feedback that efficiently shrinks the
feasible reward region.
To our knowledge, this is the first framework to jointly address
multi-environment teaching and heterogeneous feedback selection within
a unified optimization formulation.


In summary, prior work either teaches within a single MDP or infers rewards from heterogeneous feedback without structuring the teaching itself. We bridge these directions with a teacher-centric, multi-MDP framework that unifies feedback modalities through generalized behavioral equivalence classes.

\vspace{-16pt}
\section{Preliminaries}
\vspace{-12pt}
\subsection{Markov decision processes}
\vspace{-5pt}
We consider a finite Markov decision process (MDP) 
$M = (\mathcal{S}, \mathcal{A}, T, \gamma)$, 
where $\mathcal{S}$ is the state space, $\mathcal{A}$ the action space, 
$T(s' \mid s,a)$ the transition function, and $\gamma \in (0,1)$ the discount factor. 
We assume a linear reward function defined by a feature map 
$\phi : \mathcal{S} \times \mathcal{A} \to \mathbb{R}^d$, 
so that $R_w(s,a) = w^\top \phi(s,a)$, where 
$w\in \mathbb{R}^d$. 
For any policy $\pi$, we denote the discounted feature counts as
$\mu^\pi(s)=\mathbb{E}\!\left[\sum_{t=0}^\infty \gamma^t\phi(s_t,a_t)\mid s_0=s\right]$
and
$\mu^\pi(s,a)=\mathbb{E}\!\left[\sum_{t=0}^\infty \gamma^t\phi(s_t,a_t)\mid s_0=s,a_0=a\right]$. 
The corresponding value and Q-functions are 
$V_w^\pi(s)=w^\top\mu^\pi(s)$ and 
$Q_w^\pi(s,a)=w^\top\mu^\pi(s,a)$,
and we let $\pi^*(w)$ denote an optimal policy under reward parameter $w$.
\vspace{-5pt}
\subsection{Human feedback models}
\vspace{-5pt}
We model human feedback under the reward-rational choice framework \citep{jeon2020rewardrational}. A feedback instance corresponds to a choice $c \in \mathcal{C}$ from a choice set, grounded to a trajectory through a mapping $\psi:\mathcal{C}\rightarrow\Xi$. For any trajectory $\xi = (s_0,a_0,s_1,a_1,\dots)$, define its
discounted feature vector $\Phi(\xi)=\sum_{t=0}^{\infty} \gamma^t \phi(s_t,a_t)$. 
Given reward weights $w$ and rationality parameter $\beta \ge 0$, the likelihood of observing choice $c$ is 
$P(c \mid w, \beta) \propto \exp\!\big(\beta\, r_w(\psi(c))\big)$, 
where $r_w(\xi) = w^\top \Phi(\xi)$.

\noindent\textbf{Demonstrations.}
Demonstrations correspond to reward-rational choices of state–action pairs with $\mathcal{C}=\mathcal{S}\times\mathcal{A}$ and $\psi(s,a)=(s,a)$; the trajectory likelihood factorizes as
\[
P(\xi \mid w,\beta)=
\prod_{(s_t,a_t)\in\xi}
\frac{\exp(\beta Q_{w}(s_t,a_t))}
{\sum_{b\in\mathcal{A}}\exp(\beta Q_{w}(s_t,b))}.
\]
In the high-rationality limit $\beta\rightarrow\infty$, demonstrations implicitly induce preferences $\xi^+\succeq\xi^-$, where $\xi^+$ is the demonstrated trajectory $\xi$ and $\xi^-$ denotes any alternative trajectory. 

\vspace{-2pt}
\noindent\textbf{Comparisons.}
Comparisons select between trajectories
$\xi^+,\xi^-\in\Xi$ with grounding $\psi(\xi_i)=\xi_i$.
Under the Bradley–Terry model,
\[
P(\xi^+ \mid w,\beta)=
\frac{\exp(\beta r_w(\xi^+))}
{\exp(\beta r_w(\xi^+))+\exp(\beta r_w(\xi^-))}.
\]
\vspace{-2pt}
\noindent\textbf{Emergency stop (E-stop).}
Given a rollout $\xi_R$ halted at time $t$, define 
$\xi_{\mathrm{halted}} = \xi_R^{0:t}\xi_R^{t}\ldots\xi_R^{t}$, 
a trajectory of equal length that remains at state $\xi_R^t$ after halting.
The binary choice set $\mathcal{C}=\{\texttt{off},-\}$ grounds to $\psi(\texttt{off})=\xi_{\mathrm{halted}}$ and $\psi(-)=\xi_R$, yielding 
\[
P(\texttt{off}\mid w,\beta)=
\frac{\exp(\beta r_w(\xi_{\mathrm{halted}}))}
{\exp(\beta r_w(\xi_{\mathrm{halted}}))+\exp(\beta r_w(\xi_R))}.
\] 
\vspace{-2pt}
\noindent\textbf{Corrections.}
Corrections provide an improved trajectory
$\xi_{\mathrm{corr}}$ relative to a rollout $\xi_R$,
interpreted as a localized comparison:
\[
P(\xi_{\mathrm{corr}}\mid w,\beta)=
\frac{\exp(\beta r_w(\xi_{\mathrm{corr}}))}
{\exp(\beta r_w(\xi_{\mathrm{corr}}))+\exp(\beta r_w(\xi_R))}.
\]

\vspace{-4pt}
As $\beta \to \infty$, all modalities reduce to linear preference constraints $w^\top(\Phi(\xi^{+})-\Phi(\xi^{-}))\ge0$, providing the unified constraint interpretation used throughout. 
We adopt a teacher-centric view in which environments and feedback instances are selected across multiple MDPs to expose informative reward constraints. Illustrations of the feedback modalities are shown in Figure~\ref{fig:feedback_trajs}. 

\begin{figure*}[t]                
  \centering
  \begin{subfigure}[t]{0.18\textwidth}
    \centering
    \includegraphics[width=\linewidth]{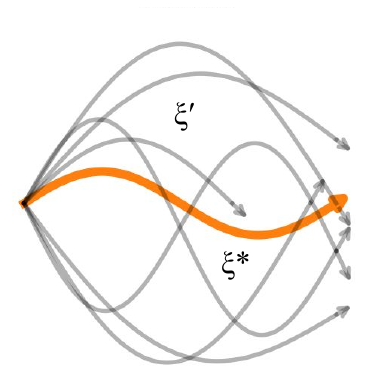}
    \subcaption{Demonstration}
  \end{subfigure}\quad
  \begin{subfigure}[t]{0.18\textwidth}
    \centering
    \includegraphics[width=\linewidth]{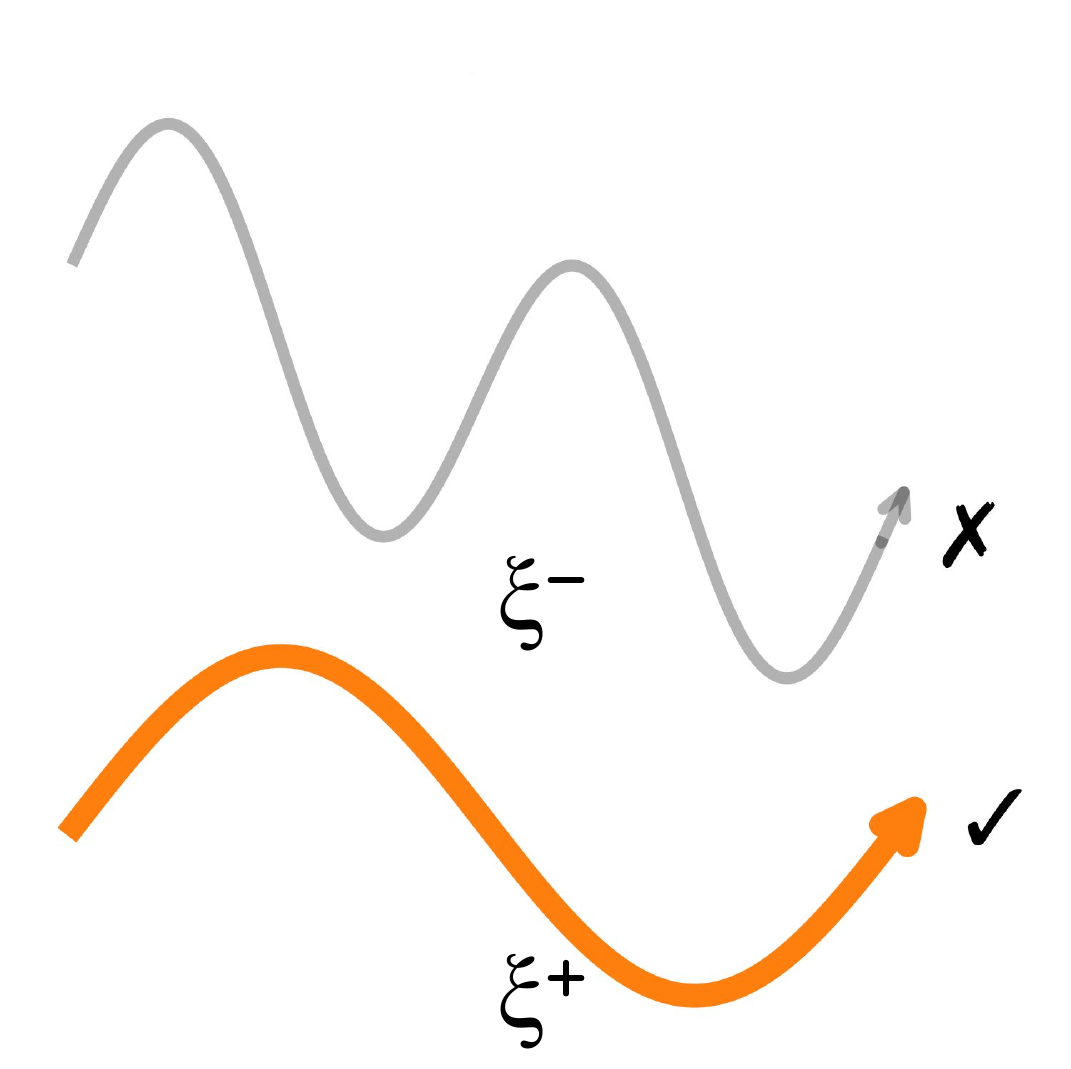}
    \subcaption{Comparison}
  \end{subfigure}\quad
    \begin{subfigure}[t]{0.18\textwidth}
    \centering
    \includegraphics[width=\linewidth]{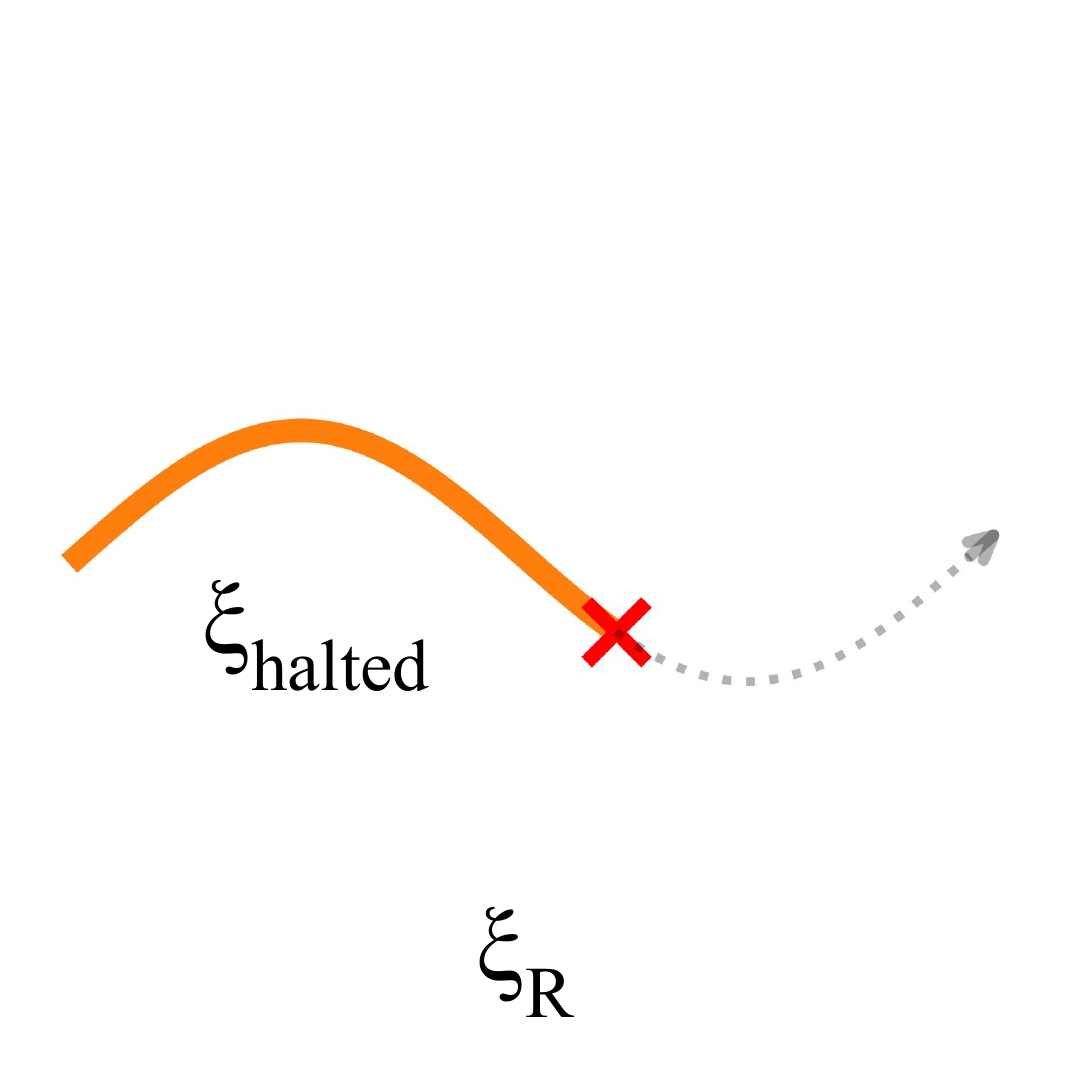}
    \subcaption{Estop}
  \end{subfigure}
  \begin{subfigure}[t]{0.18\textwidth}
    \centering
    \includegraphics[width=\linewidth]{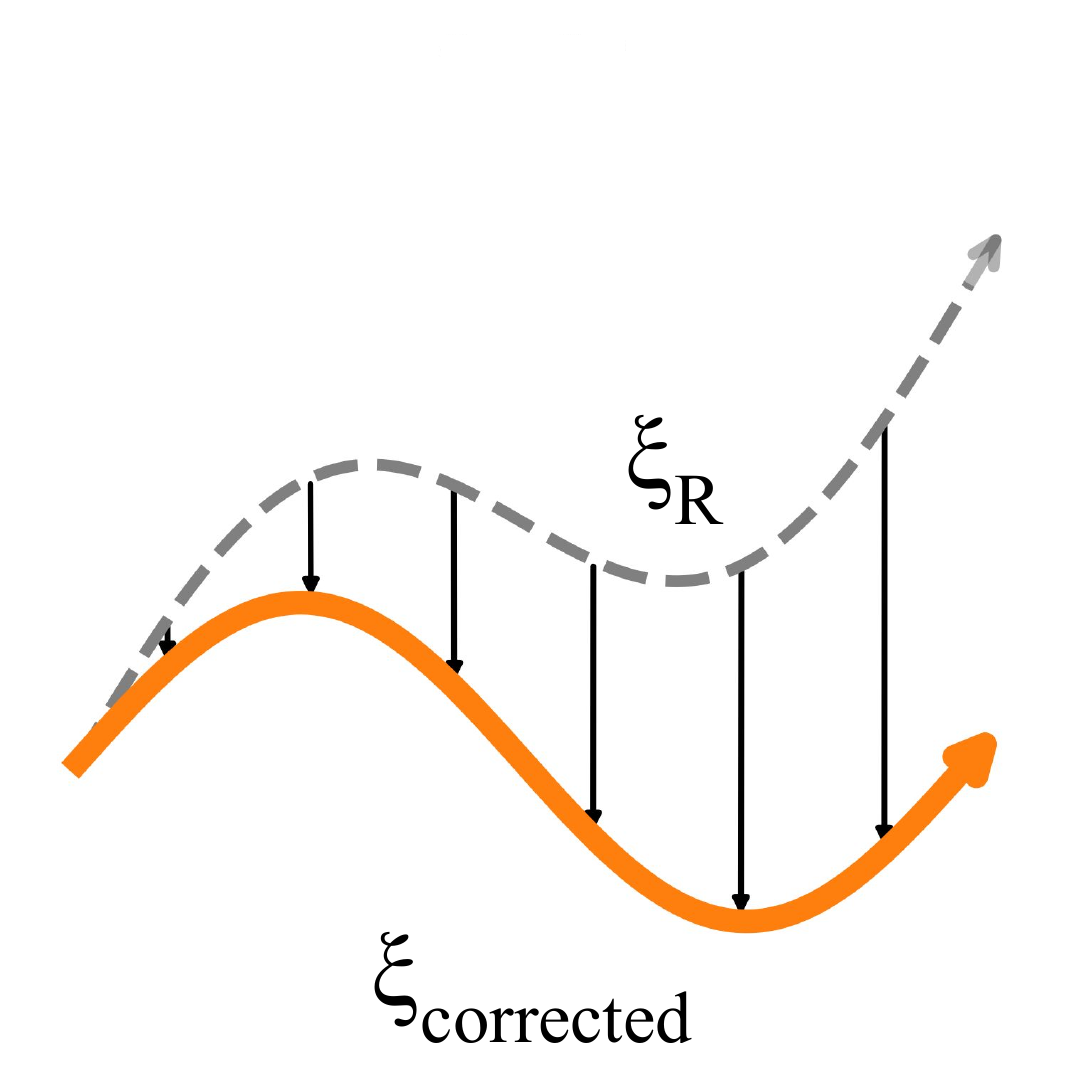}
    \subcaption{Correction}
  \end{subfigure}\quad
\caption{\textbf{Human feedback modalities as trajectory comparisons.}
Orange trajectories denote the preferred outcome.
(a) Demonstration: expert trajectory $\xi^{*}$ preferred over alternatives.
(b) Comparison: preference $\xi^+ \succ \xi^-$.
(c) E-stop: halted trajectory $\xi_{\text{halted}} \succ \xi_R$.
(d) Correction: corrected trajectory $\xi_{\mathrm{corr}} \succ \xi_R$.}

\label{fig:feedback_trajs}
\end{figure*}
\vspace{-5pt}
\section{Feedback informativeness}
\vspace{-5pt}
We analyze how different feedback modalities constrain the feasible reward space. We provide intuition for the finite-data regime and analyze idealized unlimited-data regimes to understand how each modality reduces reward ambiguity. A central tool for reasoning about reward ambiguity is the Behavioral Equivalence Class (BEC) \citep{BrownNiekum2019MachineTeachingIRL}.

\begin{definition}[Behavioral equivalence class (BEC)]
Let $\pi$ be a policy in a single MDP $M$ with linear reward
$R_w(s,a)=w^\top\phi(s,a)$.
The behavioral equivalence class of $\pi$ is
\[
\mathrm{BEC}(\pi)
=
\left\{
w \in \mathbb{R}^d \;\middle|\;
w^\top\!\big(\mu^\pi(s,\pi(s))-\mu^\pi(s,b)\big) \ge 0,
\ \forall s\in\mathcal{S},\ \forall b\in\mathcal{A}
\right\},
\]
\end{definition}
i.e., the intersection of halfspaces defined by the optimality constraints of $\pi$. 
This constraint-based view extends naturally to arbitrary feedback datasets.

\begin{definition}[Generalized behavioral equivalence class (gBEC)]
Let $D$ be a feedback dataset collected across one or more environments.
Each feedback instance induces one or more linear preference constraints of the form $w^\top\Delta\Phi\ge0$, where $\Delta\Phi=\Phi(\xi^+)-\Phi(\xi^-)$.

Let $\Delta\Phi(D)$ denote the set of all feature-difference vectors
induced by the feedback dataset $D$.
The generalized behavioral equivalence class induced by $D$ is
\begin{equation}
\mathrm{gBEC}(D)
=
\Bigl\{
w\in\mathbb{R}^d
\;\big|\;
w^\top \Delta\Phi \ge 0,
\;\forall \Delta\Phi \in \Delta\Phi(D)
\Bigr\}.
\label{eq:gbec_def}
\end{equation}
\end{definition}
\vspace{-0.8em}

\noindent\textbf{Connection to BECs.}
When $D$ consists of optimal demonstrations from a policy $\pi$ that visits all states, $\mathrm{gBEC}(D)$ coincides with the classical $\mathrm{BEC}(\pi)$, since each demonstrated state $s$ induces optimality constraints of the form
$w^\top(\mu^\pi(s,\pi(s))-\mu^\pi(s,b)) \ge 0$
for all $b\in\mathcal{A}$.
\vspace{-0.6em}

\paragraph{Feasible reward region and ambiguity.}
\vspace{-0.4em}
For any dataset $D$, the feasible reward region is $\mathrm{gBEC}(D)$. 
We quantify reward ambiguity by its volume
\[
G(D)=\mathrm{Volume}\bigl(\mathrm{gBEC}(D)\bigr),
\]
assuming $\|w\|_2\le1$ for boundedness. Smaller $G(D)$ indicates tighter reward identification.

\vspace{-5pt}
\subsection{Constraint geometry under limited and unlimited feedback}
\vspace{-5pt}
To build intuition about how feedback modalities shape the geometry of the feasible reward region, we first analyze how different forms of feedback constrain $\mathrm{gBEC}(D)$ and reduce reward ambiguity. We consider two regimes: (i) finite-data teaching with a limited feedback budget, and (ii) unlimited-data teaching, which isolates intrinsic informativeness by examining the feasible region induced by arbitrarily many feedback instances of a single type. 
\vspace{-11pt} 

\paragraph{Limited-data teaching regime.}
Figure~\ref{fig:limite_feasible_region} visualizes feasible reward regions of the MDP in Figure~\ref{fig:grid1} under tight budgets as heatmaps over reward weight space, where darker areas indicate weights that remain consistent across many randomly sampled feedback sets. Here every modality uses the same fixed budget of five queries.
Demonstrations (Fig.~\ref{fig:limited_feasible_region_demonstration}) produce the smallest $\mathrm{gBEC}(D)$ and lowest ambiguity $G(D)$. Each demonstration implicitly enforces many optimality constraints, yielding strong shrinkage of the feasible region. 
In contrast, comparisons (Fig.~\ref{fig:limited_feasible_region_pairwise}), corrections (Fig.~\ref{fig:limited_feasible_region_correction}), and E-stop feedback (Fig.~\ref{fig:limited_feasible_region_estop}) each contribute only one linear constraint per instance, resulting in progressively larger feasible regions. Corrections stay more localized thanks to their shared start-state structure, while E-stop feedback yields the largest $\mathrm{gBEC}(D)$ because its constraints are highly trajectory-specific and local. 
Demonstrations are especially powerful under limited budgets, as each piece of feedback delivers rich, multi-faceted behavioral information.
A budget-matched sweep isolating per-query efficiency via feasibility heatmaps and feasible-region volume curves is provided in Appendix~\ref{app:budget_sweep}.
\vspace{-11pt}

\paragraph{Unlimited-data teaching regime.}
We next examine the idealized limit in which a teacher can provide
arbitrarily many feedback instances of a single type.
Figure~\ref{fig:feasible_region} shows the corresponding
feasible reward regions.
In contrast to the finite-data setting, comparisons induce the
smallest $\mathrm{gBEC}(D)$ and lowest ambiguity $G(D)$,
followed by corrections, demonstrations, and E-stops.
This reversal highlights an important distinction:
while demonstrations provide strong information per example,
comparison feedback becomes more informative in the unlimited-data setting because it can
enforce global ordering constraints across the trajectory space. 
Taken together, these observations reveal a key teaching trade-off.
Different feedback modalities influence the geometry of
$\mathrm{gBEC}(D)$ in fundamentally different ways, and their relative
effectiveness depends on the available teaching budget.
This motivates the theoretical analysis in the next subsections,
where we formally characterize how feedback modalities constrain reward
ambiguity in the unlimited-data regime.

\begin{figure}[t]
  \centering
  \begin{subfigure}{0.22\textwidth}
    \centering
    \includegraphics[width=\textwidth]{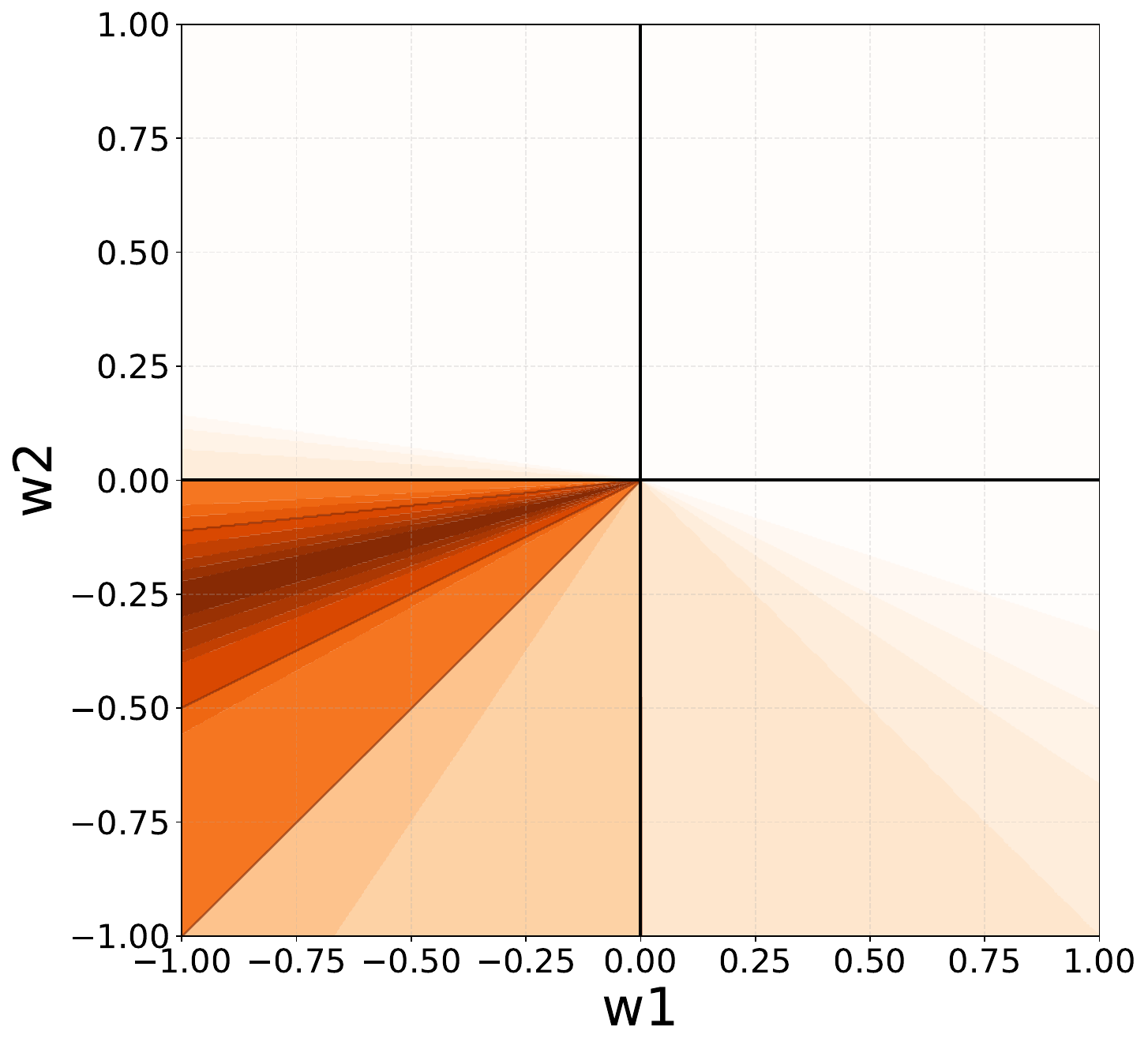}
    \caption{Comparison}
    \label{fig:limited_feasible_region_pairwise}
  \end{subfigure}
  \hfill
  \begin{subfigure}{0.22\textwidth}
    \centering
    \includegraphics[width=\textwidth]{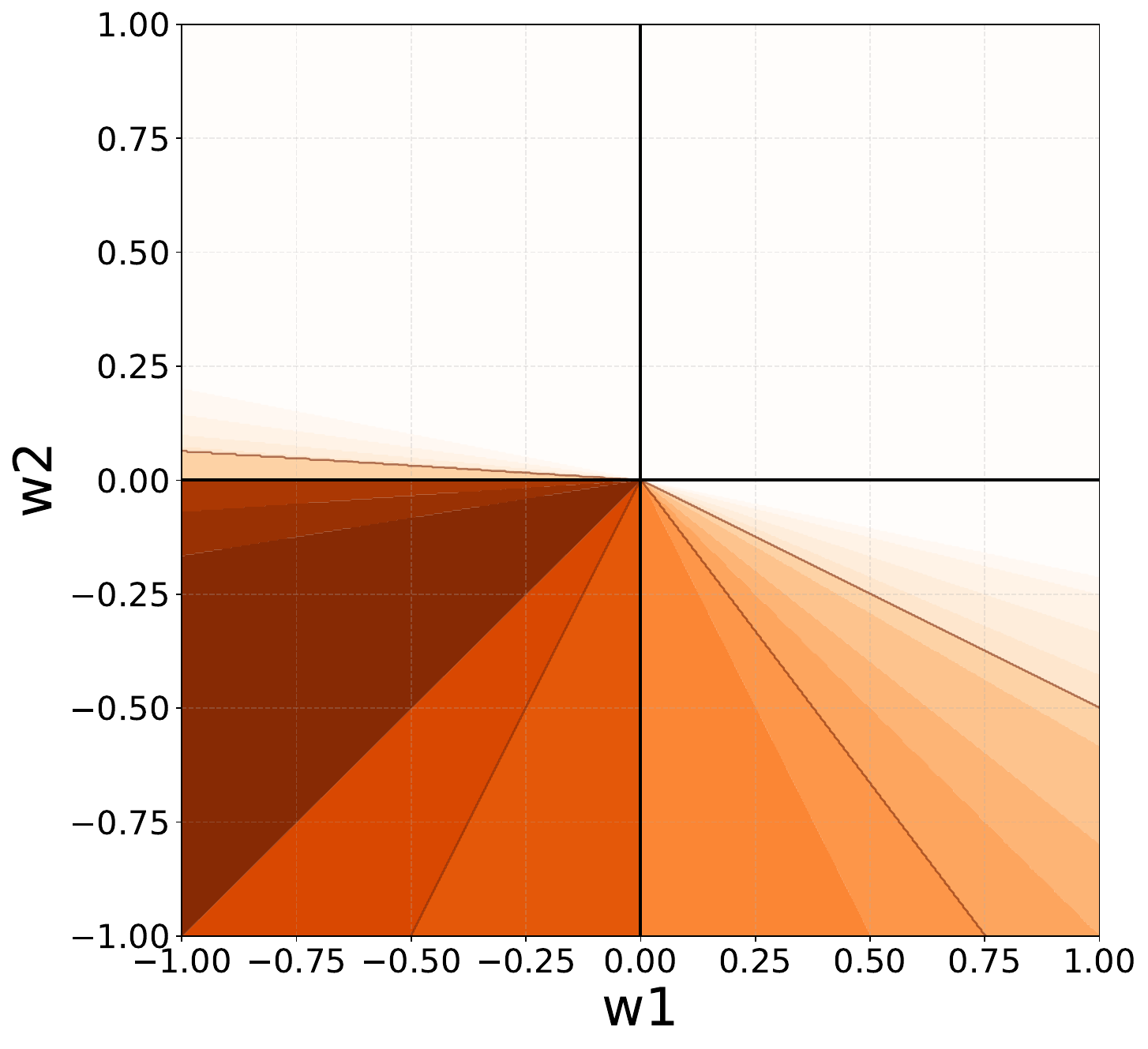}
    \caption{Correction}
    \label{fig:limited_feasible_region_correction}
  \end{subfigure}
  \hfill
  \begin{subfigure}{0.22\textwidth}
    \centering
    \includegraphics[width=\textwidth]{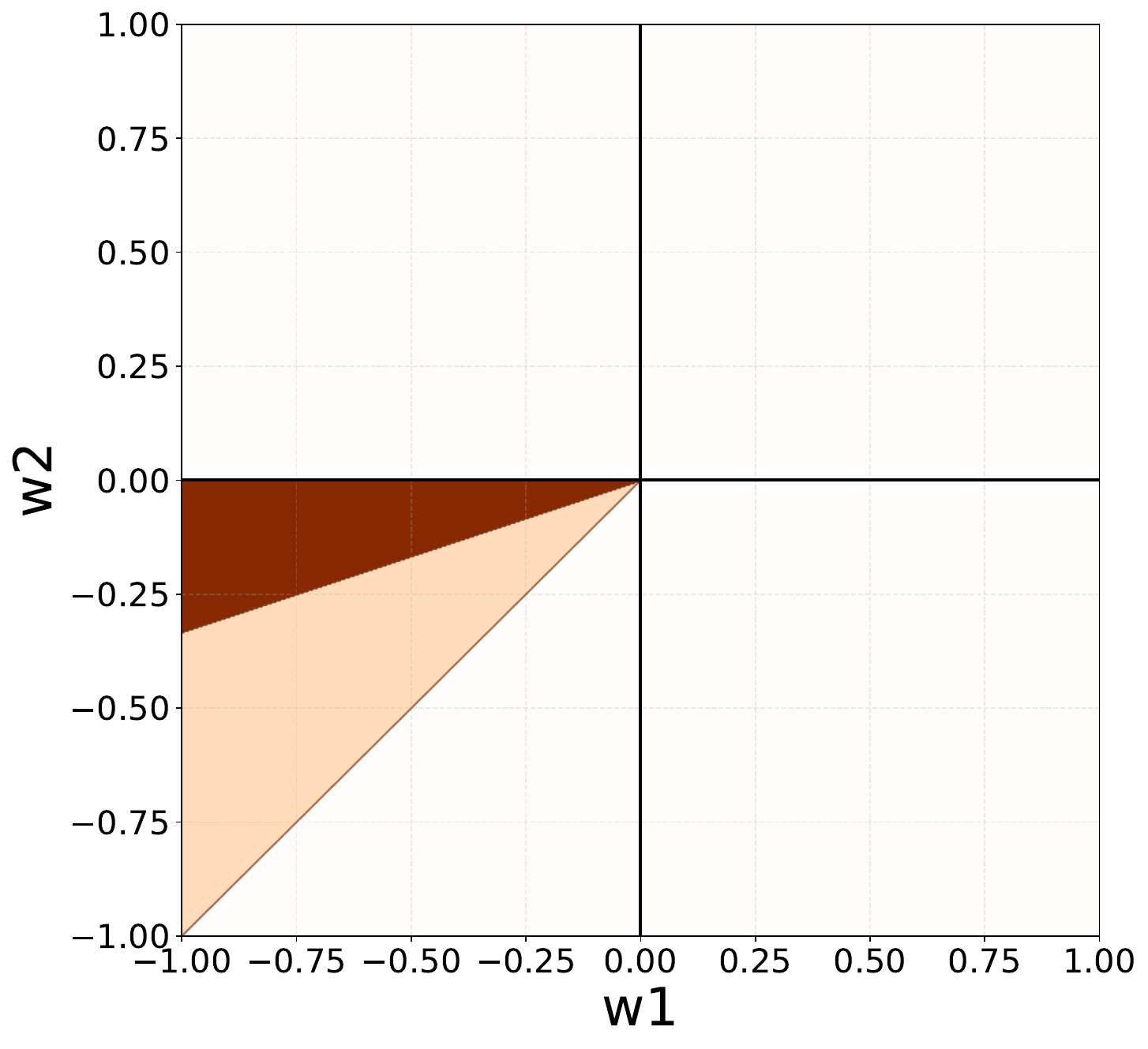}
    \caption{Demonstration}
    \label{fig:limited_feasible_region_demonstration}
  \end{subfigure}
  \hfill
  \begin{subfigure}{0.22\textwidth}
    \centering
    \includegraphics[width=\textwidth]{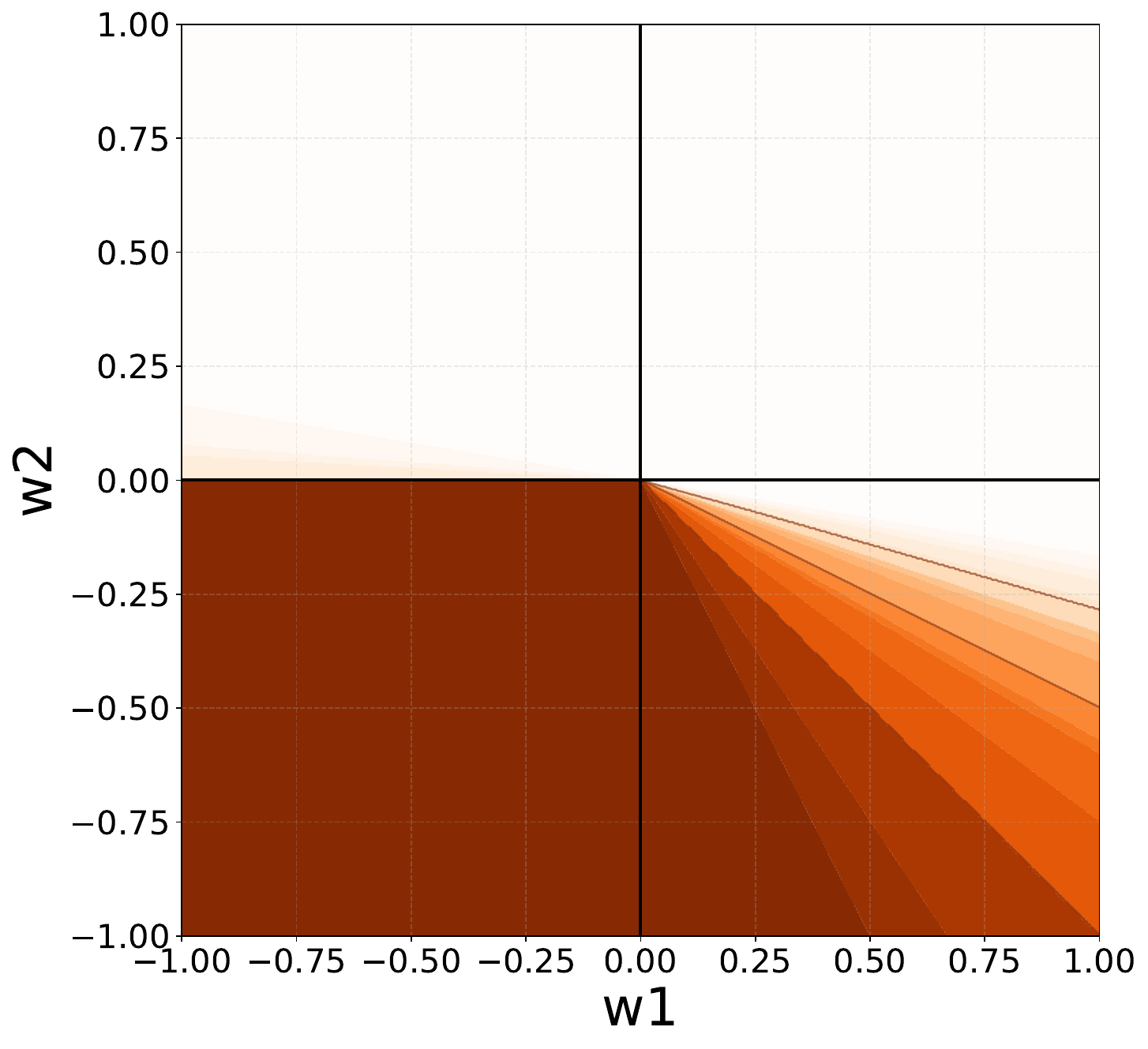}
    \caption{E-stop}
    \label{fig:limited_feasible_region_estop}
  \end{subfigure}

\caption{Finite-data feasible reward regions under different types of feedback.
Heatmaps visualize the empirical feasibility probability over the reward
space---darker colors indicate weight vectors that remain feasible
across many randomly sampled feedback sets.
The heatmap illustrates how different feedback modalities constrain
the generalized behavioral equivalence class $\mathrm{gBEC}(D)$ and
affect reward ambiguity $G(D)$ in the data-limited regime (corresponding to the layout in Figure~\ref{fig:grid1}). }

\label{fig:limite_feasible_region}

\end{figure}

\begin{figure}[t]
  \centering
  \begin{subfigure}{0.22\textwidth}
    \centering
    \includegraphics[width=\textwidth]{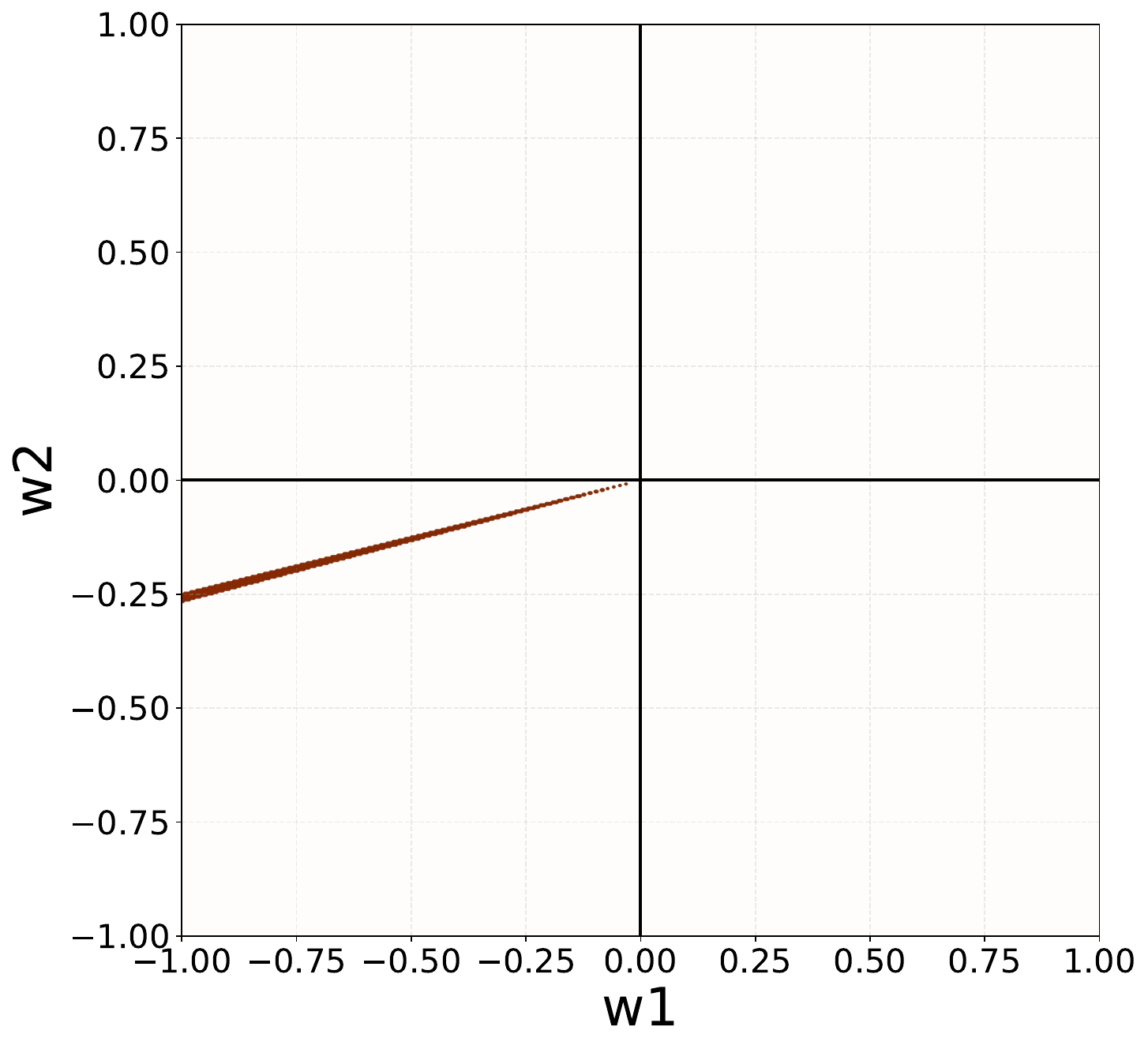}
    \caption{Comparison}
    \label{fig:feasible_region_pairwise}
  \end{subfigure}
  \hfill
  \begin{subfigure}{0.22\textwidth}
    \centering
    \includegraphics[width=\textwidth]{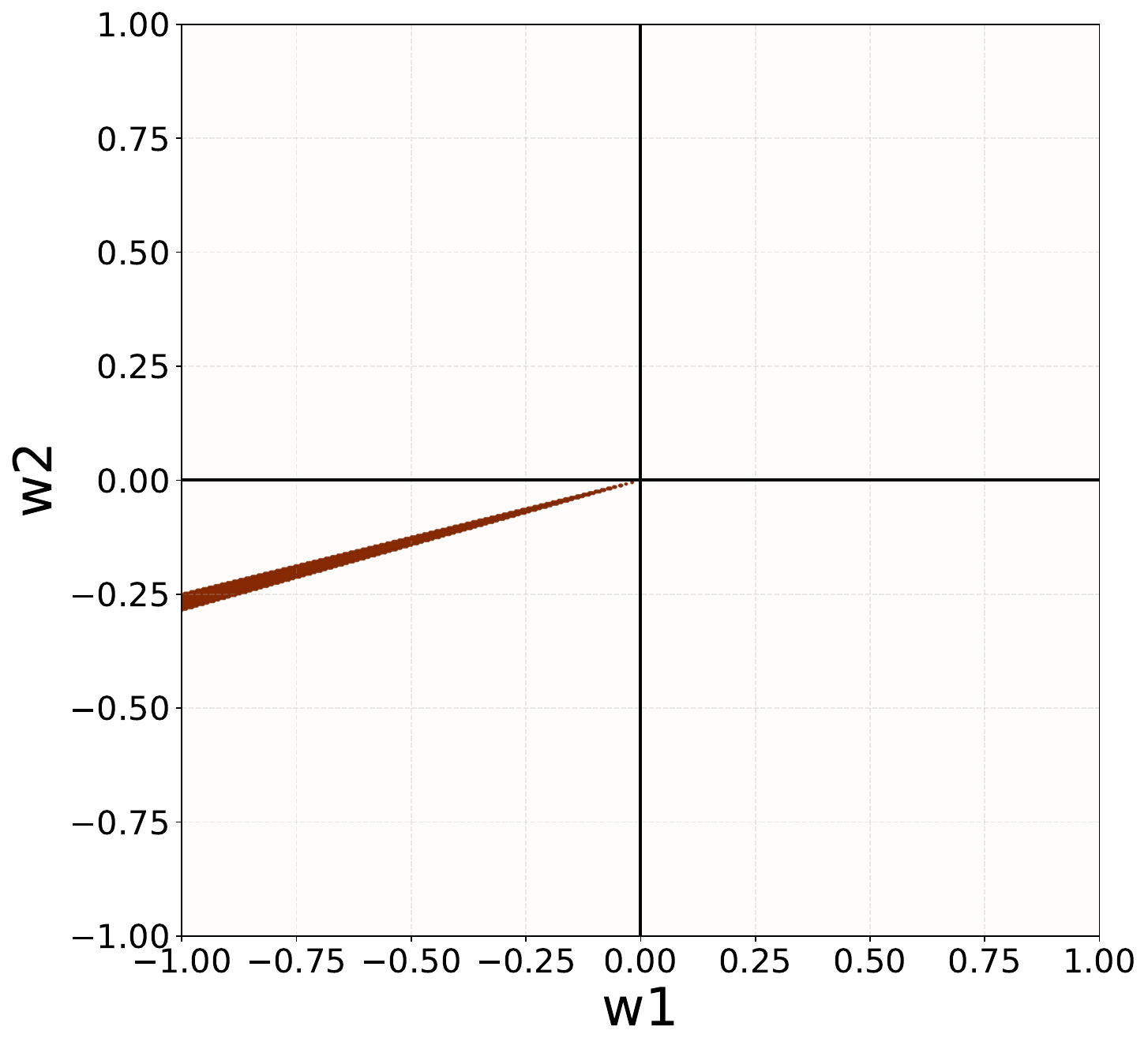}
    \caption{Correction}
    \label{fig:feasible_region_correction}
  \end{subfigure}
  \hfill
  \begin{subfigure}{0.22\textwidth}
    \centering
    \includegraphics[width=\textwidth]{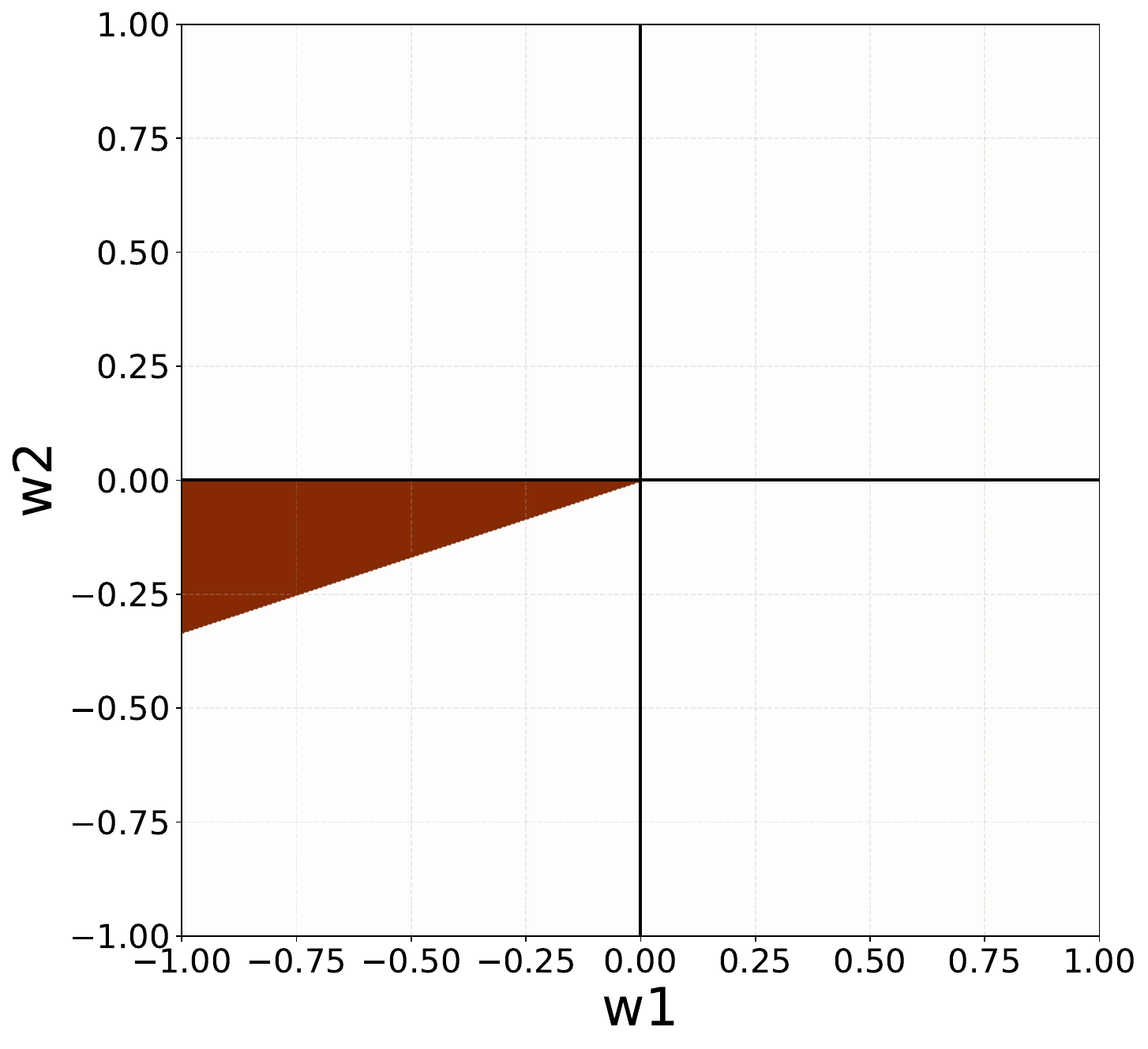}
    \caption{Demonstration}
    \label{fig:feasible_region_demonstration}
  \end{subfigure}
  \hfill
  \begin{subfigure}{0.22\textwidth}
    \centering
    \includegraphics[width=\textwidth]{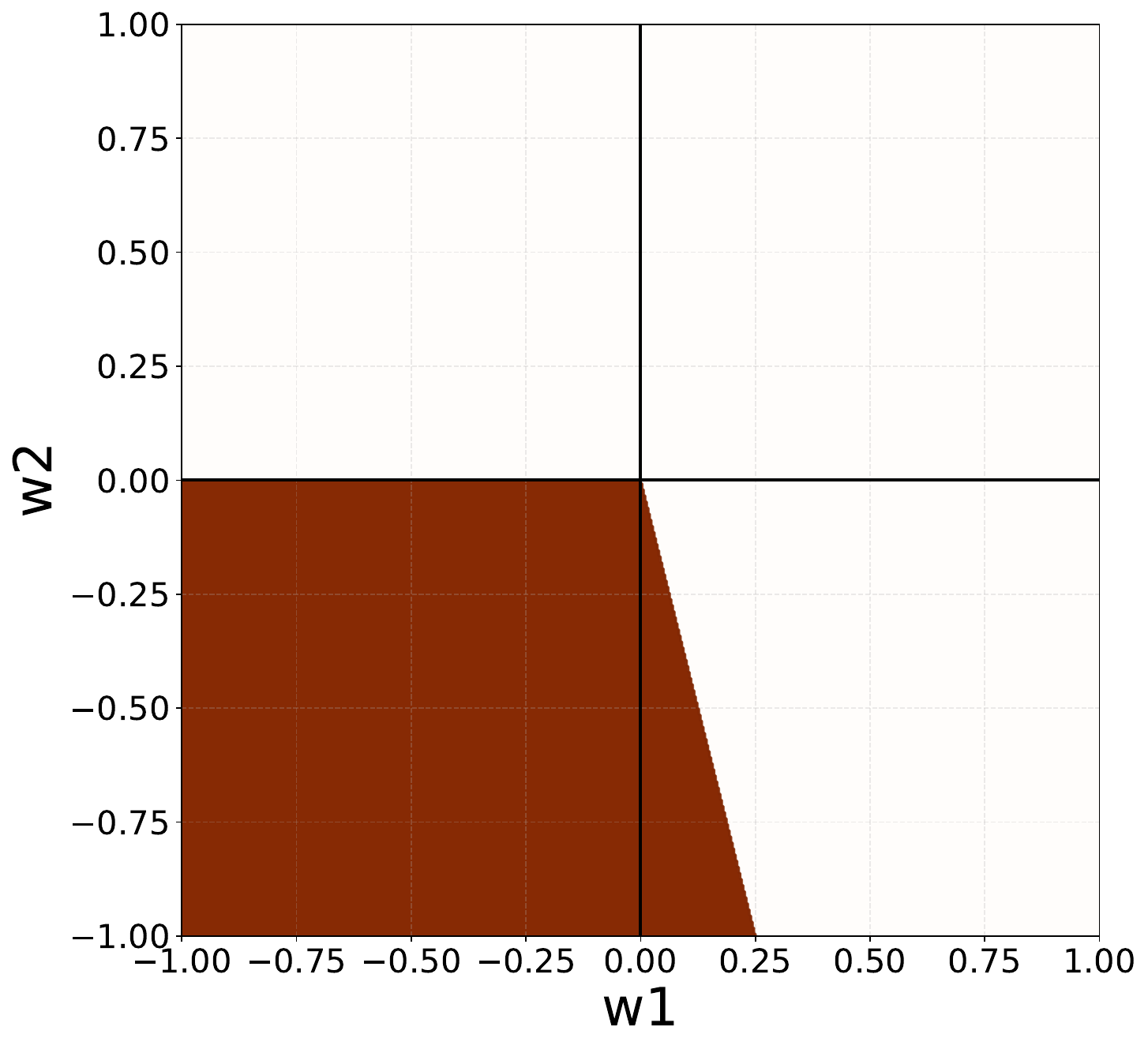}
    \caption{E-stop}
    \label{fig:feasible_region_estop}
  \end{subfigure}
  
    \caption{Feasible reward regions $\mathrm{gBEC}(D)$ for different feedback modalities. Subfigures (a)–(d) show how infinite comparisons, corrections, demonstrations, and E-stops, respectively, shape the reward feasibility region. The black dot indicates the ground-truth reward parameter (corresponding to the layout in Figure~\ref{fig:grid1}).}

\label{fig:feasible_region}
\end{figure}

\section{Theoretical analysis of feedback modalities}

The geometric analysis above showed that different feedback modalities reduce reward ambiguity in qualitatively different ways. 
We now formalize these differences by analyzing how each modality constrains the feasible reward region in the unlimited-data regime.

Let $\Xi$ denote the trajectory space of environment $M$, i.e., the set of all trajectories that can arise in $M$. 
Throughout this section we assume the teacher can obtain arbitrarily many feedback instances of a given modality. 
We use comparisons as a reference modality. 
The comparison modality compares any two trajectories $\xi_i,\xi_j \in \Xi$, allowing ordering constraints across the entire trajectory space. 
Therefore, in the unlimited-data regime, a complete set of preferences consistent with the true reward $w^\star$ induces a full ordering over $\Xi$ and yields the feasible region
\vspace{-5pt}
\[
H_{\text{comparison}} =
\Bigl\{
w \in \mathbb{R}^d : \;
w^\top(\Phi(\xi_i)-\Phi(\xi_j)) \ge 0
\;\forall\, \xi_i,\xi_j \in \Xi
\text{ with }
w^{\star\top}\Phi(\xi_i) > w^{\star\top}\Phi(\xi_j)
\Bigr\}.
\]
\vspace{-5pt}
We compare this region to those induced by demonstrations, corrections, and E-stop feedback.

\vspace{-3pt}

\subsection{Comparison vs. demonstrative feedback}
\vspace{-5pt}

As mentioned previously, demonstrations induce implicit constraints involving trajectories from the same initial state. 
Comparison feedback, in contrast, can rank any two trajectories in $\Xi$, imposing ordering constraints across the entire trajectory space.

\begin{proposition}[Comparisons strictly reduce reward ambiguity compared to demonstrations]
\label{prop:pairwise_vs_demo}

Let $\Xi^\star \subseteq \Xi$ denote the trajectories that are optimal under $w^\star$. 
Demonstrations induce the feasible region
\vspace{-5pt}
\[
H_{\text{demo}} =
\Bigl\{
w \in \mathbb{R}^d : \;
w^\top(\Phi(\xi^\star)-\Phi(\xi')) \ge 0 \;
\forall \xi^\star \in \Xi^\star,\;
\forall \xi' \in \Xi \text{ with the same initial state as } \xi^\star
\Bigr\}.
\]

\vspace{-3pt}
Then
\vspace{-5pt}
\[
H_{\text{comparison}} \subsetneq H_{\text{demo}},
\quad\text{and}\quad
G(H_{\text{comparison}}) < G(H_{\text{demo}}).
\]

\end{proposition}

\vspace{-6pt}

\subsection{Comparison vs. corrective feedback}
\vspace{-5pt}

Corrective feedback compares a corrected trajectory $\xi_{\mathrm{corr}}$ with the agent's rollout $\xi_R$, where both trajectories share the same initial state.

\begin{proposition}[Comparisons strictly reduce reward ambiguity compared to corrective feedback]
\label{prop:comparison_vs_correction}

Corrective feedback induces the feasible region
\vspace{-5pt}
\[
H_{\text{correction}} =
\Bigl\{
w \in \mathbb{R}^d : \;
w^\top(\Phi(\xi_{\mathrm{corr}})-\Phi(\xi_R)) \ge 0 \;
\forall \xi_{\mathrm{corr}}, \xi_R \in \Xi
\text{ with the same initial state}
\Bigr\}.
\]
\vspace{-3pt}
Then
\vspace{-5pt}
\[
H_{\text{comparison}} \subsetneq H_{\text{correction}},
\quad\text{and}\quad
G(H_{\text{comparison}}) < G(H_{\text{correction}}).
\]

\end{proposition}

\vspace{-6pt}

\subsection{Comparison vs. E-stop feedback}
\vspace{-5pt}

E-stop feedback compares a halted prefix $\xi^{0:t}$ with the continuation of the same trajectory $\xi$, penalizing undesirable future behavior. 
Because both share the same prefix, the feature difference $\Phi(\xi^{0:t})-\Phi(\xi)$ depends only on states visited after time $t$ along the same trajectory. 
Consequently, E-stop constraints are confined to the feature subspace induced by that trajectory.



\begin{lemma}[E-stop constraints are trajectory-local]
\label{lemma:distincttraj}
Fix a trajectory $\xi \in \Xi$.
Any E-stop constraint comparing $\xi^{0:t}$ with $\xi$ lies in the subspace
\[
\operatorname{span}
\{\phi(s,a) \mid (s,a) \text{ appears in } \xi\}.
\]
In contrast, a pairwise comparison between $\xi$ and any $\xi' \in \Xi$ 
that visits a state-action pair $(s',a')$ with $\phi(s',a')$ linearly 
independent of this subspace produces a feature difference that lies 
outside it.
\end{lemma}

\vspace{-4pt}

\begin{proposition}[Comparisons strictly reduce reward ambiguity compared to E-stop feedback]
\label{prop:estop_comparison}

E-stop feedback induces the feasible region
\vspace{-5pt}
\[
H_{\text{E-stop}} =
\Bigl\{
w \in \mathbb{R}^d : 
w^\top(\Phi(\xi^{0:t})-\Phi(\xi)) \ge 0 \;
\forall \xi \in \Xi,\;
\forall t
\Bigr\}.
\]
\vspace{-3pt}
Then
\vspace{-5pt}
\[
H_{\text{comparison}} \subsetneq H_{\text{E-stop}},
\quad\text{and}\quad
G(H_{\text{comparison}}) < G(H_{\text{E-stop}}).
\]

\end{proposition}

\section{Single-MDP limitations: environment-dependent reward identifiability}
\vspace{-5pt}
The previous section characterized the intrinsic informativeness of different feedback modalities in the unlimited-data regime within a single MDP. 
In realistic settings, however, teaching must operate under finite feedback budgets and produce rewards that remain effective across different environments. Even with unlimited feedback available in one MDP, the constraints induced by that particular environment may fail to uniquely identify the ground-truth reward when the agent is deployed in other environments. 
We now formalize a key limitation of single-MDP teaching: reward identifiability depends on the environment. 

\begin{theorem}[Single-MDP ambiguity]
\label{thm:single_mdp_ambiguity}
Let $\mathcal{M}=\{M_1,M_2\}$ be two finite MDPs that share a common 
feature map $\phi:\mathcal{S}\times\mathcal{A}\rightarrow\mathbb{R}^d$ 
and a common ground-truth reward parameter $w^\star \in \mathbb{R}^d$, 
where rewards take the linear form $r(s,a)=w^{\star\top}\phi(s,a)$.
For any MDP $M_k$, let $\mathcal{W}(M_k)$ denote the generalized behavioral
equivalence class (gBEC), i.e.,
the set of reward parameters consistent with \emph{all possible}
idealized feedback constraints obtainable within $M_k$.
Define the feature-difference span induced by $M_k$ as
\[
\mathcal{V}_k
=\operatorname{span}\{\Delta\Phi(\xi^+,\xi^-)\;:\;
\xi^+,\xi^- \text{ are feasible trajectories in } M_k\},
\]
where $\Delta\Phi(\xi^+,\xi^-)=\Phi(\xi^+)-\Phi(\xi^-)$
and $\Phi(\xi)$ denotes the feature expectation of trajectory $\xi$.
Assume that the induced spans satisfy 
$\mathcal{V}_1 \subsetneq \operatorname{span}(\mathcal{V}_1\cup\mathcal{V}_2)$. 
Then there exists a reward parameter $w\neq w^\star$ such that 
$w \in \mathcal{W}(M_1)$ but $w \notin \mathcal{W}(M_1)\cap\mathcal{W}(M_2)$. 
\end{theorem}

The proof is provided in Appendix~\ref{app:proofs}.

The theorem shows that reward identifiability depends critically on the span of feature differences achievable by feasible trajectories in each MDP. When a single MDP does not span all reward-relevant directions, residual ambiguity persists even with unlimited idealized feedback.
For a concrete example, consider two gridworld MDPs that share the same 2D feature representation (white/gray cells) and ground-truth linear reward $  w^\star  $, but differ in layout, cell arrangement, and terminal state position (see Figures~\ref{fig:grid1} and \ref{fig:grid2}). These layout differences produce distinct sets of feasible trajectories and thus distinct spans of achievable feature differences, leading to different generalized behavioral equivalence classes.
The resulting feasible reward regions appear in Figures~\ref{fig:feasible1} and \ref{fig:feasible2}: $  \mathrm{gBEC}_1 \subsetneq \mathrm{gBEC}_2  $. Every reward consistent with all optimal demonstrations in MDP 1 remains feasible in MDP 2 (but not vice versa). 
Teaching only in $M_2$ leaves reward ambiguity that can cause generalization failures to $M_1$, while teaching in $M_1$ recovers a reward that transfers across both environments. This asymmetry stems purely from differences in the environment layout. Appendix~\ref{app:layout_gallery} confirms this layout-dependent ambiguity across $200$ randomly sampled layouts.

These observations reveal that single-MDP teaching is frequently insufficient for reward generalization across environments, motivating the use of diverse training environments.
Under realistic finite feedback budgets, the situation becomes even more challenging: the teacher faces two tightly coupled problems—(i) identifying environments whose combined constraint spans resolve the remaining ambiguity, and (ii) selecting the most informative feedback instances within those environments under a limited total query budget. 
This naturally motivates a hierarchical approach, in which environment selection activates complementary constraint subspaces and feedback selection efficiently covers those subspaces. 
We next formalize this two-stage strategy as a hierarchical set-cover optimization problem and introduce hierarchical set cover optimal teaching (HSCOT), an algorithm for multi-environment teaching with heterogeneous feedback.

\begin{figure*}[t]
  \centering
  \begin{subfigure}[t]{0.23\textwidth}
    \centering
    \includegraphics[width=\linewidth]{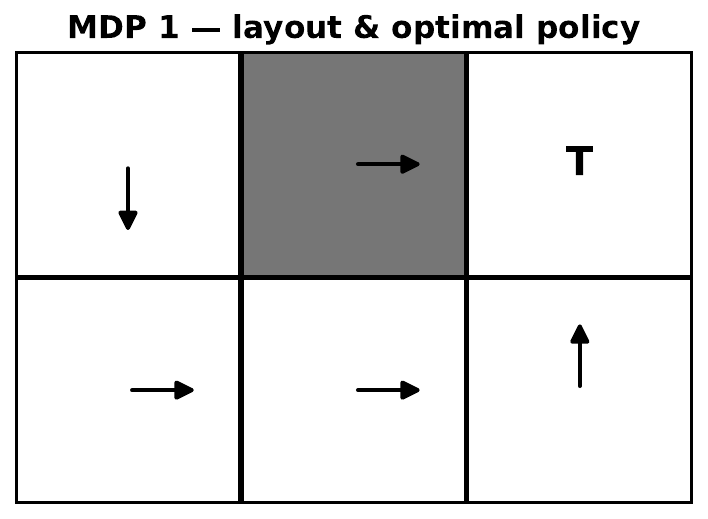}
    \caption{MDP 1}
    \label{fig:grid1}
  \end{subfigure}
  \hfill
  \begin{subfigure}[t]{0.23\textwidth}
    \centering
    \includegraphics[width=\linewidth]
    {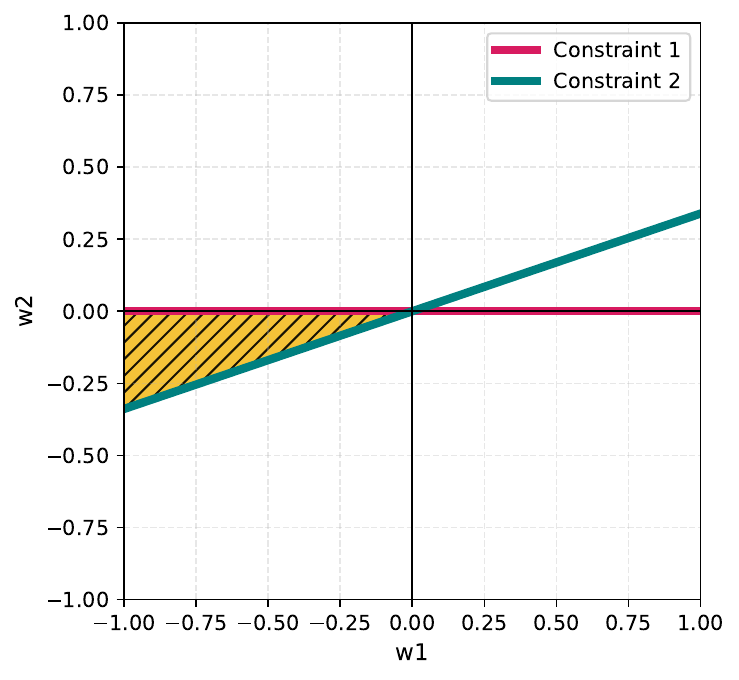}
    
    \caption{gBEC in MDP 1}
    \label{fig:feasible1}
  \end{subfigure}
  \hfill
  \begin{subfigure}[t]{0.23\textwidth}
    \centering
    \includegraphics[width=\linewidth]{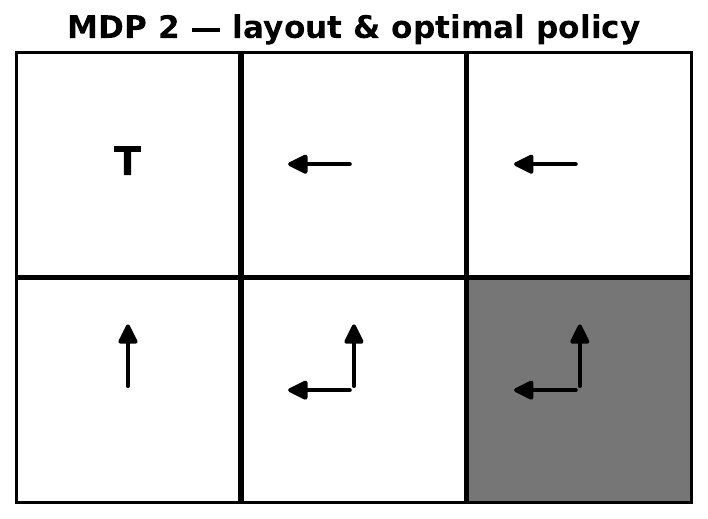}
    \caption{MDP 2}
    \label{fig:grid2}
  \end{subfigure}
  \hfill
  \begin{subfigure}[t]{0.23\textwidth}
    \centering
    \includegraphics[width=\linewidth]{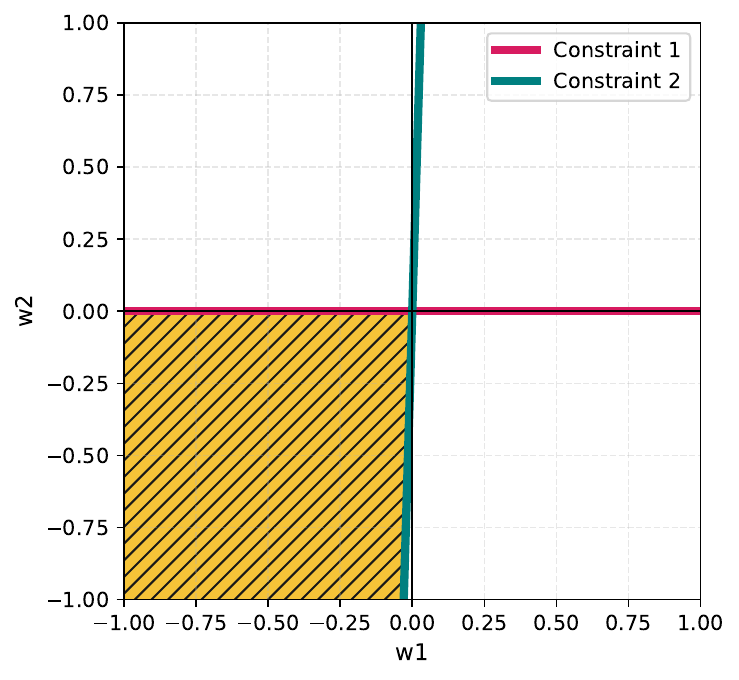}
    \caption{gBEC in MDP 2}
    \label{fig:feasible2}
  \end{subfigure}
    \caption{{\textbf{Environment-dependent reward ambiguity.}
      Two gridworld MDPs share the same reward features but differ in layout; each MDP has two features, shown as white and gray cells, and \textbf{T} marks the terminal state. Arrows in (a) and (c) show an optimal policy.
      The feasible reward region in MDP~2 is strictly larger than in MDP~1 ($\mathrm{gBEC}_1 \subsetneq \mathrm{gBEC}_2$), illustrating how environment structure affects reward identifiability.}}
  
\label{fig:reward_ambiguity_asymmetry}
\end{figure*}

\vspace{-5pt}
\section{Hierarchical Set Cover Optimal Teaching (HSCOT)}
\label{sec:hscot}
\vspace{-5pt}
We introduce \textit{Hierarchical Set Cover Optimal Teaching (HSCOT)}, a machine teaching framework that learns reward functions robust to changes in environment dynamics by strategically selecting both environments and feedback queries. 
The key idea is that different environments expose different reward constraints through their feasible trajectory spaces. 
HSCOT exploits this structure by first selecting environments whose dynamics reveal complementary reward constraints and then selecting feedback instances within those environments that efficiently cover those constraints.
\vspace{-5pt}
\subsection{Generalized teaching formulation}
\label{sec:generalized_teaching}
\vspace{-5pt}
We formulate teaching a reward learning agent across multiple MDPs with heterogeneous feedback as a machine teaching problem. 
This generalizes the single-MDP, demonstration-only setting of \citet{BrownNiekum2019MachineTeachingIRL} to a multi-environment setting where the teacher selects both environments and feedback instances to efficiently constrain the reward and promote generalization.

Let $\mathcal{M}=\{M_1,\dots,M_K\}$ be a family of MDPs sharing a feature map $\phi:\bigcup_k(\mathcal{S}_k\times\mathcal{A}_k)\to\mathbb{R}^d$, where $\mathcal{S}_k$ and $\mathcal{A}_k$ denote the state and action spaces of $M_k$. 
Let $\mathcal{M}_{\mathrm{train}},\mathcal{M}_{\mathrm{test}}\subseteq\mathcal{M}$ denote training and evaluation environments. 
A feedback atom is $x=(k,f,c)$, where $M_k\in\mathcal{M}_{\mathrm{train}}$, $f\in\{\text{demo},\text{comp},\text{estop},\text{corr}\}$ denotes the modality, and $c$ is the queried choice. 
A dataset $D=\{x_1,\dots,x_N\}$ is a collection of such atoms. 
In the high-rationality limit ($\beta\to\infty$), corresponding to an expert teacher, each atom induces linear preference constraints $w^\top\Delta\Phi\ge0$, where $\Delta\Phi$ is the discounted feature-count difference between preferred and dispreferred trajectories. 
These constraints define the feasible reward region
\[
\mathcal{W}(D)=\{w\in\mathbb{R}^d \mid w^\top\Delta\Phi\ge0,\ \forall\Delta\Phi\text{ induced by }x\in D\}.
\]

We evaluate a learned reward $w$ using the average performance gap 
relative to the ground-truth reward $w^\star$ across a given set of 
evaluation environments, $\mathcal{M}'$. 
Let $V^{\pi}_{k}(w)=\mathbb{E}\left[\sum_{t=0}^{\infty}\gamma^t 
r_w(s_t,a_t)\mid \pi, M_k\right]$ denote the expected discounted 
return of policy $\pi$ in $M_k$ under reward parameter $w$. 
We then define
\begin{equation}
\mathrm{Loss}(w^\star,w, \mathcal{M}')=\frac{1}{|\mathcal{M}'|}
\sum_{k\in\mathcal{M}'}\left(V_{k}^{\pi_k^*(w^)}(w^\star)
-V_{k}^{\pi_k(w)}(w^\star)\right)
\label{eq:loss_multi}
\end{equation}
where $\pi_k^*(w)$ denotes the optimal policy in $M_k$ under reward $w$. 
We assume an expert teacher who knows $w^\star$ and the dynamics of all $M_k\in\mathcal{M}_{\mathrm{train}}$, but not the evaluation environments. 
A teaching instance is therefore specified by $\mathcal{T}=(\mathcal{M},\mathcal{M}_{\mathrm{train}},\mathcal{M}_{\mathrm{test}},\phi,\mathbb{R}^d,w^\star,\varepsilon)$, where $\varepsilon$ is a target loss tolerance.

We want to solve for a teaching set $D$ that minimizes some notion of teaching cost and also minimizes loss in Eq. (\ref{eq:loss_multi}). However, as is typical of machine teaching problems in general~\citep{zhu2018overview}, this is a difficult multi-objective optimization problem. Minimizing loss on an unseen set of MDPs is also very difficult without making any assumptions on these unseen test environments. To make our approach tractable, we focus on a constrained form of the machine teaching problem~\citep{zhu2018overview}, where we focus on minimizing loss on the known train set of MDPs. 
This mirrors the standard empirical risk minimization paradigm in machine learning, where performance is optimized on a training set to minimize error on unseen test instances. Under the assumption that training and evaluation environments are drawn i.i.d. from the same distribution, minimizing loss on the training MDPs provides a principled surrogate for minimizing loss on unseen environments.

The teacher aims to find a minimal subset of environments 
$\mathcal{K}\subseteq\mathcal{M}_{\mathrm{train}}$ and, within each selected 
environment, provide a minimal set of feedback instances such that the 
inferred reward generalizes across $\mathcal{M}_{\mathrm{train}}$. 

Let $\hat w=\mathrm{IRL}(D)$ denote the inferred reward. 
We define the hierarchical teaching cost as $\mathrm{TeachingCost}(D)=(|\mathrm{MDPs}(D)|,|D|)$ and solve
\vspace{-1pt}
\begin{align}
\min_D \quad & \mathrm{TeachingCost}(D) \\
\text{s.t.}\quad & \mathrm{Loss}(w^\star,\hat w, \mathcal{M}_{\rm train})\le\varepsilon, \\
& \hat w=\mathrm{IRL}(D).
\end{align}
\vspace{-1pt}
The minimization is lexicographic: first minimizing the number of environments used and then the number of feedback queries.

\vspace{-5pt}
\subsection{Constraint universe}
\label{sec:constraint_universe}
\vspace{-5pt}

To operationalize the teaching objective, we represent feedback in terms of the reward constraints it induces. 
For each training MDP $M_k \in \mathcal{M}_{\mathrm{train}}$, let $\mathcal{X}_k$ denote the finite set of candidate feedback atoms available in that environment. 
Each atom $x \in \mathcal{X}_k$ induces a set of discounted feature-difference vectors $\Delta\Phi(x)=\{\Delta\Phi_1,\dots,\Delta\Phi_m\}$.

We define the \emph{constraint universe} $\mathcal{U}$ as the set of all distinct linear preference constraints inducible from the training environments, 
$\mathcal{U}=\bigcup_{k\in\mathcal{M}_{\mathrm{train}}}\bigcup_{x\in\mathcal{X}_k}\Delta\Phi(x)$, where duplicates are removed so each feature-difference direction appears only once. 
For each environment $M_k$, we similarly define $\mathcal{U}_k=\bigcup_{x\in\mathcal{X}_k}\Delta\Phi(x)\subseteq\mathcal{U}$, representing the constraint directions accessible within that environment.

This representation highlights a key structural property of the teaching problem: environment dynamics determine which reward constraints are reachable, while individual feedback atoms instantiate those constraints. This separation provides the basis for the hierarchical teaching strategy introduced next.
\vspace{-5pt}
\subsection{Hierarchical teaching strategy}
\label{sec:hierarchical_teaching}
\vspace{-5pt}

The constraint representation above exposes a natural hierarchical structure in the teaching problem. 
This observation allows the teaching objective to be decomposed into two coupled decisions: selecting informative environments and selecting feedback instances within those environments.
\vspace{-5pt}
\paragraph{Outer level (environment selection).}
The first stage selects a minimal subset of environments $\mathcal{K}\subseteq\mathcal{M}_{\mathrm{train}}$ such that $\bigcup_{k\in\mathcal{K}}\mathcal{U}_k=\mathcal{U}$, identifying environments whose trajectory spaces collectively expose all reward-relevant constraint directions.
\vspace{-5pt}
\paragraph{Inner level (atom selection).}
Given $\mathcal{K}$, the second stage selects a minimal dataset $D$ satisfying $\bigcup_{x\in D}\Delta\Phi(x)=\mathcal{U}$ with $\mathrm{MDPs}(D)\subseteq\mathcal{K}$, determining the feedback instances required to realize those constraints. 
\vspace{-5pt}
\paragraph{Greedy hierarchical approximation (HSCOT).}
Both stages correspond to set-cover problems over the finite constraint universe $\mathcal{U}$. 
We therefore take advantage of submodularity~\citep{nemhauser1978analysis} and employ greedy selection procedures that iteratively choose the environment or feedback atom providing the largest marginal increase in uncovered constraint directions. 
The resulting algorithm, HSCOT, decomposes teaching into two stages: 
an outer stage that selects environments by their marginal constraint 
coverage, and an inner stage that selects feedback atoms within those environments.
The full pseudocode for the outer-stage environment selection, the inner-stage atom selection, and the combined hierarchical procedure appears in Algorithms~\ref{alg:env_selection}--\ref{alg:hscot}.
After teaching, the learner estimates the reward $\hat{w} = \mathrm{IRL}(D)$ 
and deploys the resulting policy; we then evaluate generalization 
performance using the loss in Equation~\eqref{eq:loss_multi}.

\vspace{-5pt}
\section{Experimental results}
\label{sec:experiments}
\vspace{-5pt}
We evaluate whether hierarchical multi-environment teaching (HSCOT)
improves reward identifiability and policy generalization under limited feedback budgets.

\vspace{-5pt}
\paragraph{Domains and setup.}
We evaluate on two domains with deterministic transitions and fixed discount factor. In a 6$\times$6 GridWorld, we generate 50 MDPs sharing the same linear reward $w^\star$ (normalized so $\|w^\star\|_2\le 1$), but differing in layout and transition structure; feature vectors are sampled per environment. 
In LavaMiniGrid~\citep{chevalier2023minigrid}, we use four hand-designed features (\textit{distance-to-goal}, \textit{on-lava}, \textit{adjacent-to-lava}, \textit{per-step cost}) with corresponding ground-truth weights $w^\star = [-1.0,\, -8.0,\, -2.0,\, -0.05]$, and generate 50 MDPs with the same ground-truth reward but different layouts (see Appendix~\ref{app:env_maps} for visualizations of the environments). 
For generalization, 20\% of MDPs (10 of 50) are held out and never seen during teaching. 
All results are averaged over 10 independent seeds, where each seed is a full replicate of the entire pipeline---environment generation, ground-truth reward, train/held-out split, and feedback draws.
\vspace{-6pt}

\paragraph{Baseline and Reward Learning.} We compare HSCOT against a \textit{Uniform Teaching} baseline that samples feedback atoms
uniformly across environments under the same global query budget as HSCOT.
Reward recovery uses max-margin inverse reinforcement learning
\citep{abbeel2004apprenticeship}.

\vspace{-5pt}

\paragraph{Evaluation metrics.}
We report \emph{held-out regret} using the average performance gap (Eq.~\eqref{eq:loss_multi}) between the policy induced by the learned reward $\hat{w}$ and the ground-truth optimal policy, averaged over held-out MDPs. 
We also report \emph{constraint coverage}, the fraction of the universal constraint set $\mathcal{U}$ covered by the selected feedback atoms $D$:
\vspace{-6pt}
\[
\text{coverage}(D) = \frac{|\mathcal{U}(D)|}{|\mathcal{U}|},
\]
where $\mathcal{U}$ is the set of all distinct linear preference constraints $w^\top(\Phi(\xi^+)-\Phi(\xi^-))\ge 0$ inducible by any feedback atom and modality across all training MDPs, and $\mathcal{U}(D)$ is the subset induced by $D$. Coverage serves as a proxy for reduction in reward ambiguity (shrinkage of the generalized behavioral equivalence class).
\paragraph{Held-out regret.}
Single-modality results are shown in Figures~\ref{fig:regret_gridworld} 
and \ref{fig:regret_lavagrid}. 
In GridWorld, HSCOT achieves zero held-out regret across all modalities, 
while uniform teaching incurs substantially higher regret in every case, 
confirming that strategic environment selection is critical even under 
identical budgets. 
In LavaMiniGrid, HSCOT again attains near-zero regret across modalities, 
whereas uniform teaching leaves residual regret throughout. 
Under uniform teaching, demonstrations remain comparatively effective 
since each one implicitly imposes many optimality constraints, and 
corrections benefit from steering rollouts toward the goal, yielding 
informative goal-directed constraints. 
Notably, E-stop outperforms comparisons under uniform querying in 
LavaMiniGrid: halting points concentrate on imminent lava entry---the 
dominant reward direction---whereas randomly sampled trajectory pairs 
rarely differ in lava contact, making many comparison queries 
uninformative. 
Mixed-feedback results (Figures~\ref{fig:regret_tie_gridworld} 
and \ref{fig:regret_tie_lavagrid}) show that combining demonstrations 
with another modality allows HSCOT to reach near-zero regret in both 
domains, as complementary constraint directions across environments 
tighten the feasible reward region and yield rewards that transfer 
reliably to unseen MDPs.

\begin{figure}[t]
\centering

\begin{subfigure}{0.24\textwidth}
\includegraphics[width=\linewidth]{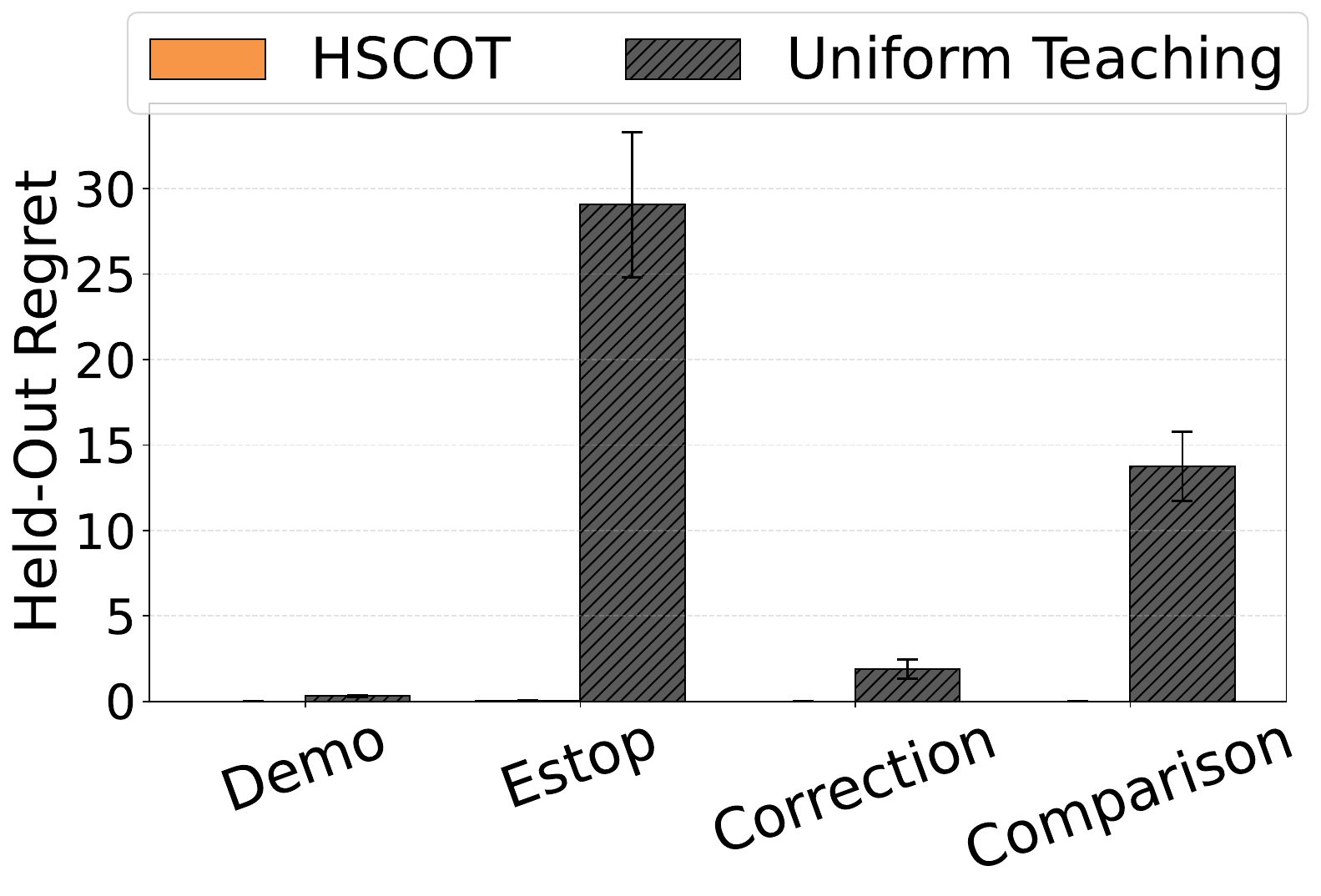}
\caption{GridWorld}
\label{fig:regret_gridworld}
\end{subfigure}
\hfill
\begin{subfigure}{0.24\textwidth}
\includegraphics[width=\linewidth]{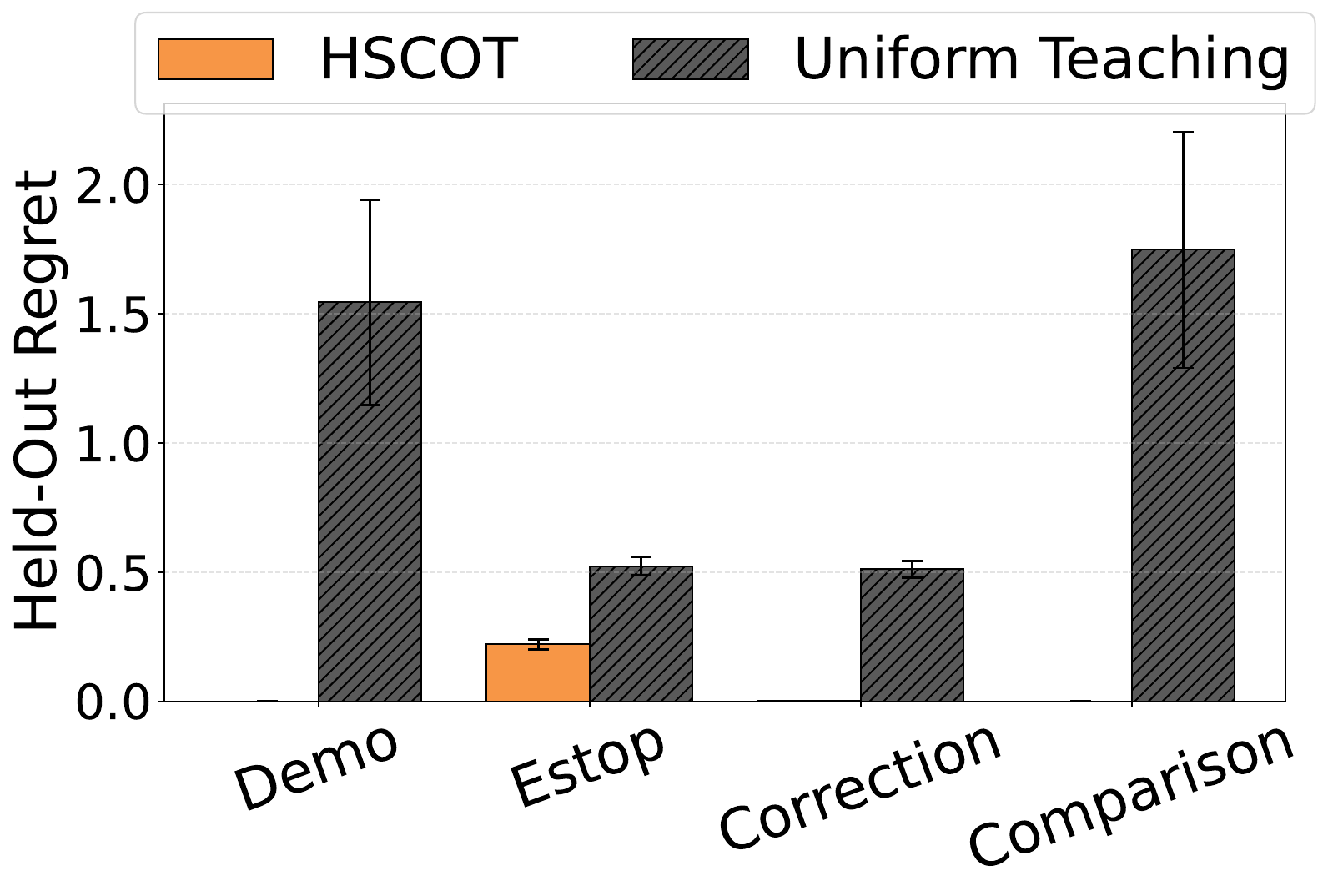}
\caption{LavaMiniGrid}
\label{fig:regret_lavagrid}
\end{subfigure}
\hfill
\begin{subfigure}{0.24\textwidth}
\includegraphics[width=\linewidth]{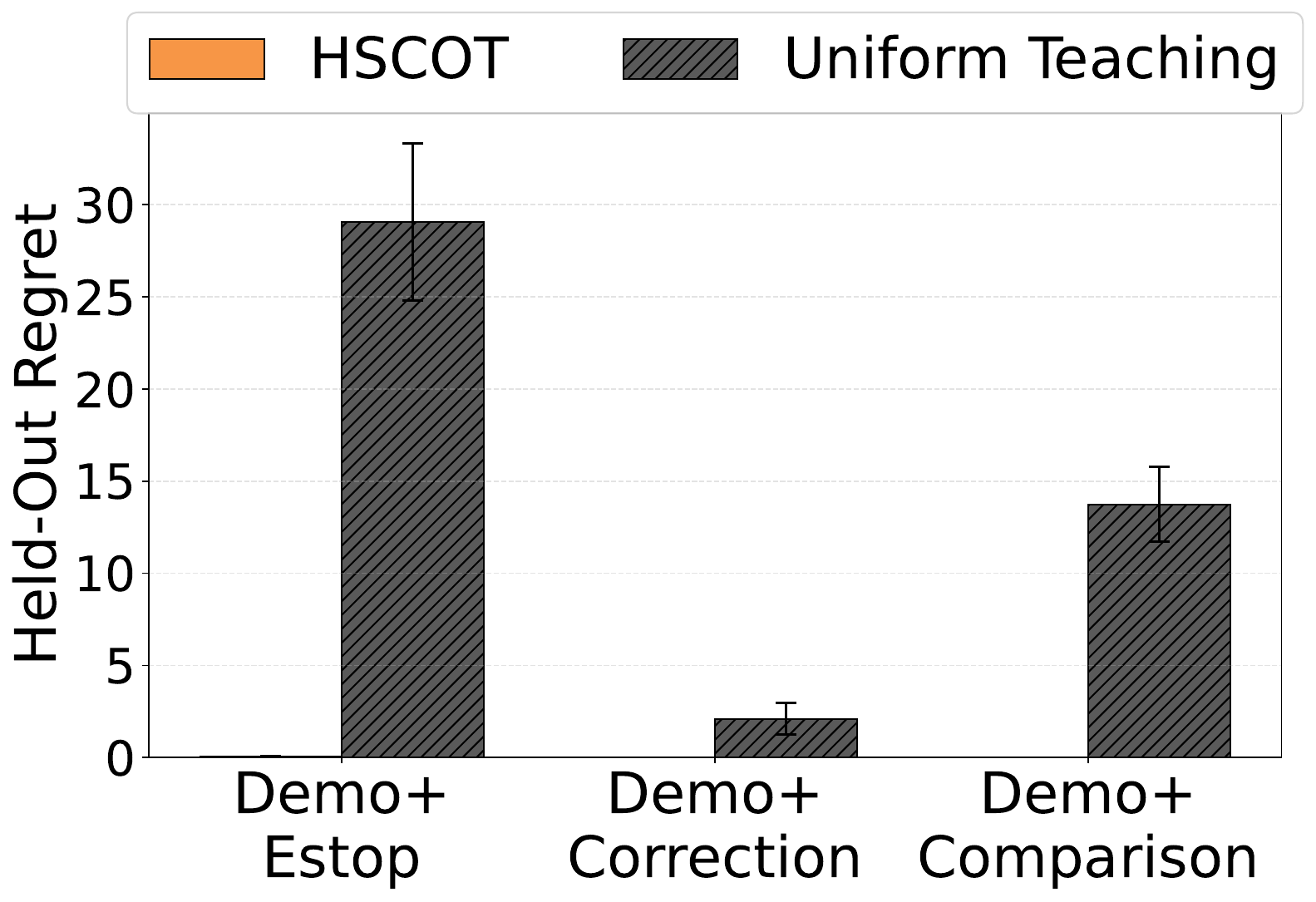}
\caption{GridWorld}
\label{fig:regret_tie_gridworld}
\end{subfigure}
\hfill
\begin{subfigure}{0.24\textwidth}

\includegraphics[width=\linewidth]{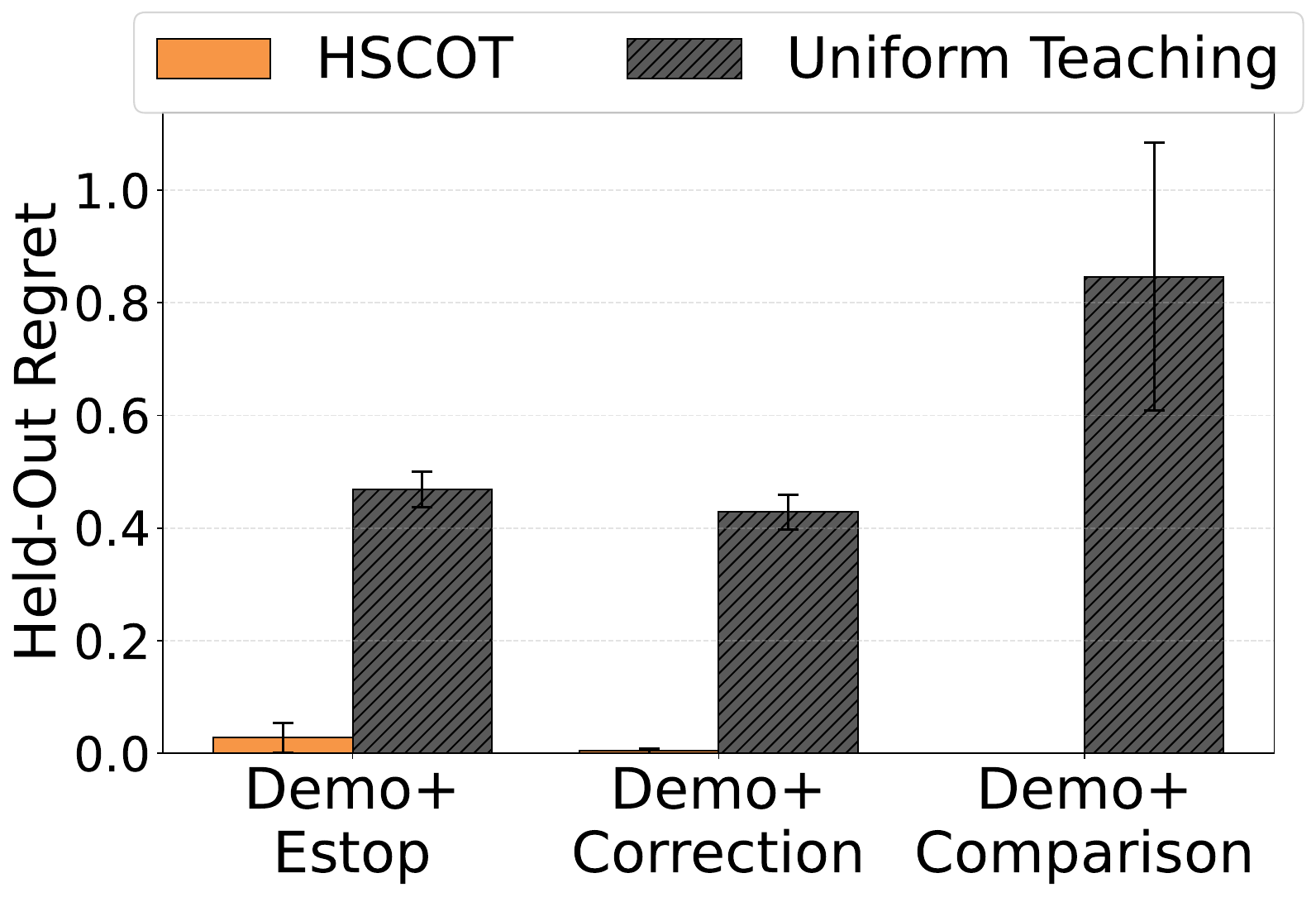}
\caption{LavaMiniGrid}
\label{fig:regret_tie_lavagrid}
\end{subfigure}

\caption{\textbf{Held-out regret} (lower is better) averaged over 10 random seeds on 20\% held-out MDPs. Bars show mean regret and error bars denote standard error across seeds. Subfigures (a,b) show results for single-modality feedback. Subfigures (c,d) show settings where demonstrations are combined with an additional modality. HSCOT (solid bars) consistently outperforms uniform teaching (dashed bars) across both domains and feedback settings. Overall, rewards learned by HSCOT generalize better to unseen environments.}

\label{fig:heldout_main}

\end{figure}
\vspace{-8pt}
\paragraph{Constraint coverage.}
Figure~\ref{fig:coverage_main} shows the fraction of the universal constraint set covered by the selected feedback atoms, computed over the training MDPs. 
HSCOT's two-level greedy approach (environment selection followed by atom selection) consistently achieves complete coverage across both GridWorld and LavaMiniGrid and all feedback modalities. In contrast, uniform teaching recovers only a small fraction of the universal set---even when using demonstrations, where each individual demonstration implicitly imposes many constraints. 
This large and consistent gap demonstrates two important points: (1) HSCOT can recover all constraints that can be generated across the training MDPs, and (2) intelligent selection of environments remains critical for constraint coverage, even when each piece of feedback is highly informative on its own. Uniform sampling, by spreading queries thinly across uninformative MDPs, misses most of the complementary constraints needed for tight reward identification.
\vspace{0pt}

\vspace{-5pt}
\begin{figure}[t]
  \centering
  \begin{subfigure}{0.24\textwidth}
    \includegraphics[width=\linewidth]{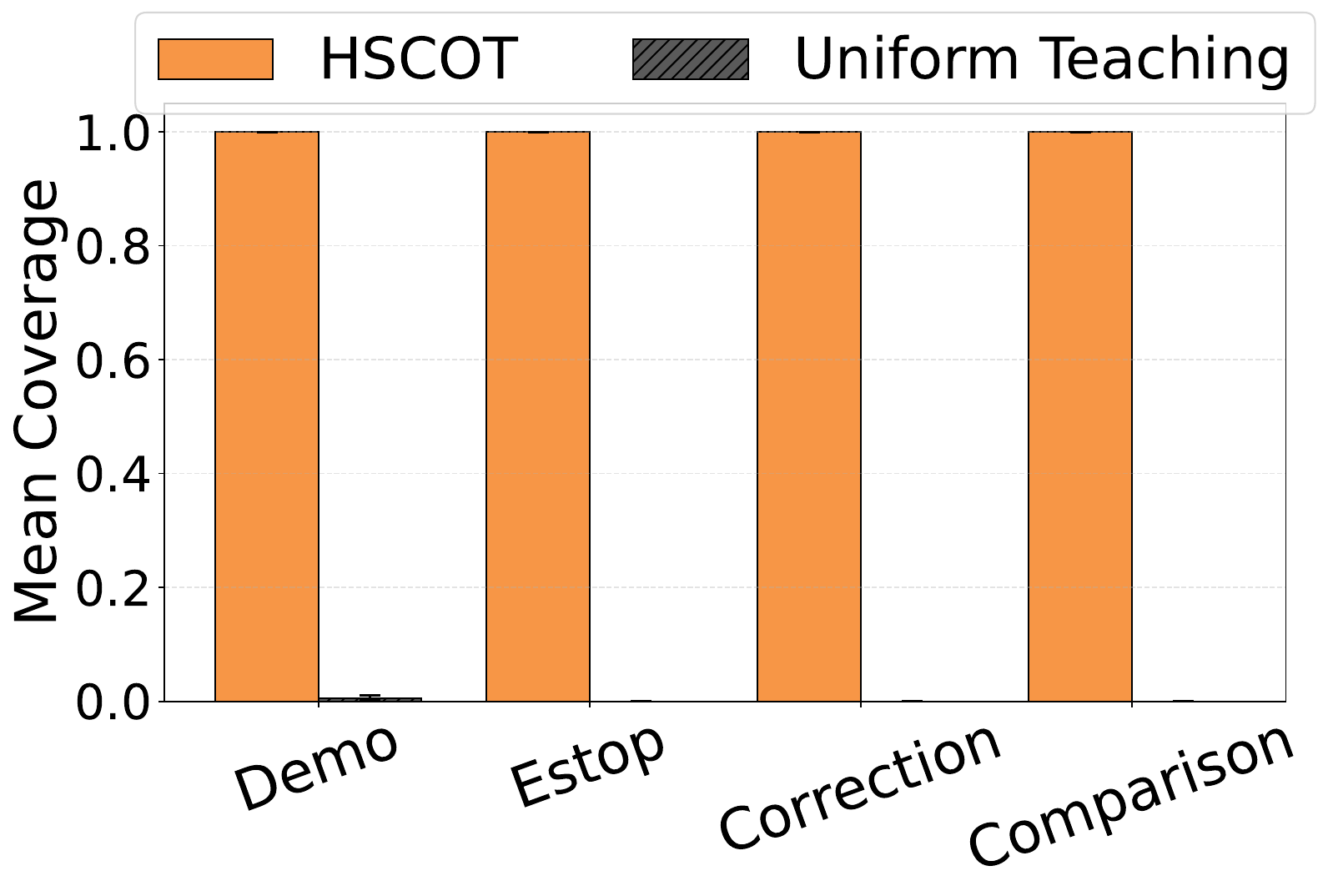}
    \caption{GridWorld}
    \label{fig:coverage_main_gridworld}
  \end{subfigure}
  \hfill
  \begin{subfigure}{0.24\textwidth}
    \includegraphics[width=\linewidth]{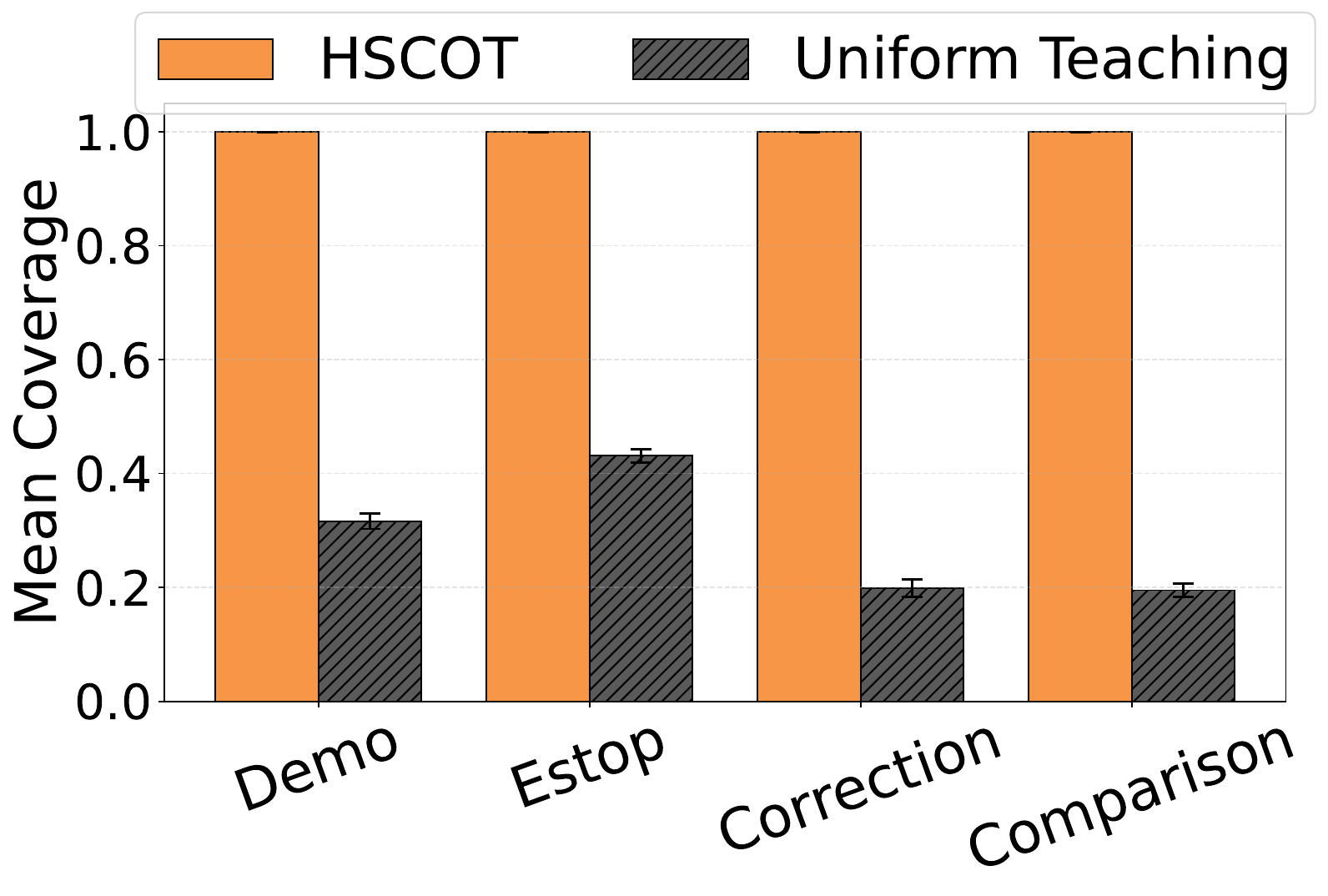}
    \caption{LavaMiniGrid}
    \label{fig:coverage_main_minigrid}
  \end{subfigure}
  \hfill
  \begin{subfigure}{0.24\textwidth}
    \includegraphics[width=\linewidth]{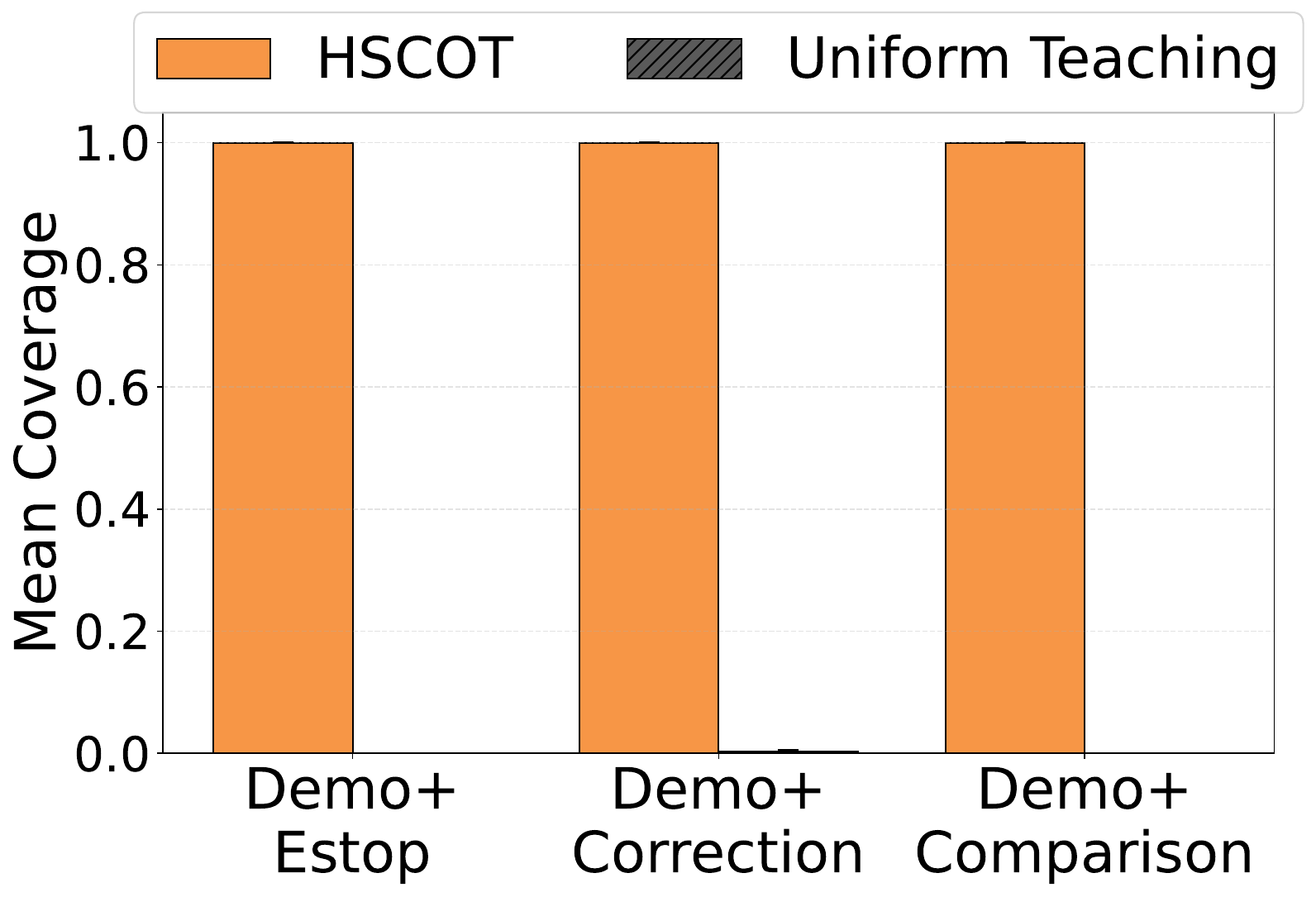}
    \caption{GridWorld}
    \label{fig:coverage_main_placeholder1}
  \end{subfigure}
  \hfill
  \begin{subfigure}{0.24\textwidth}
    \includegraphics[width=\linewidth]{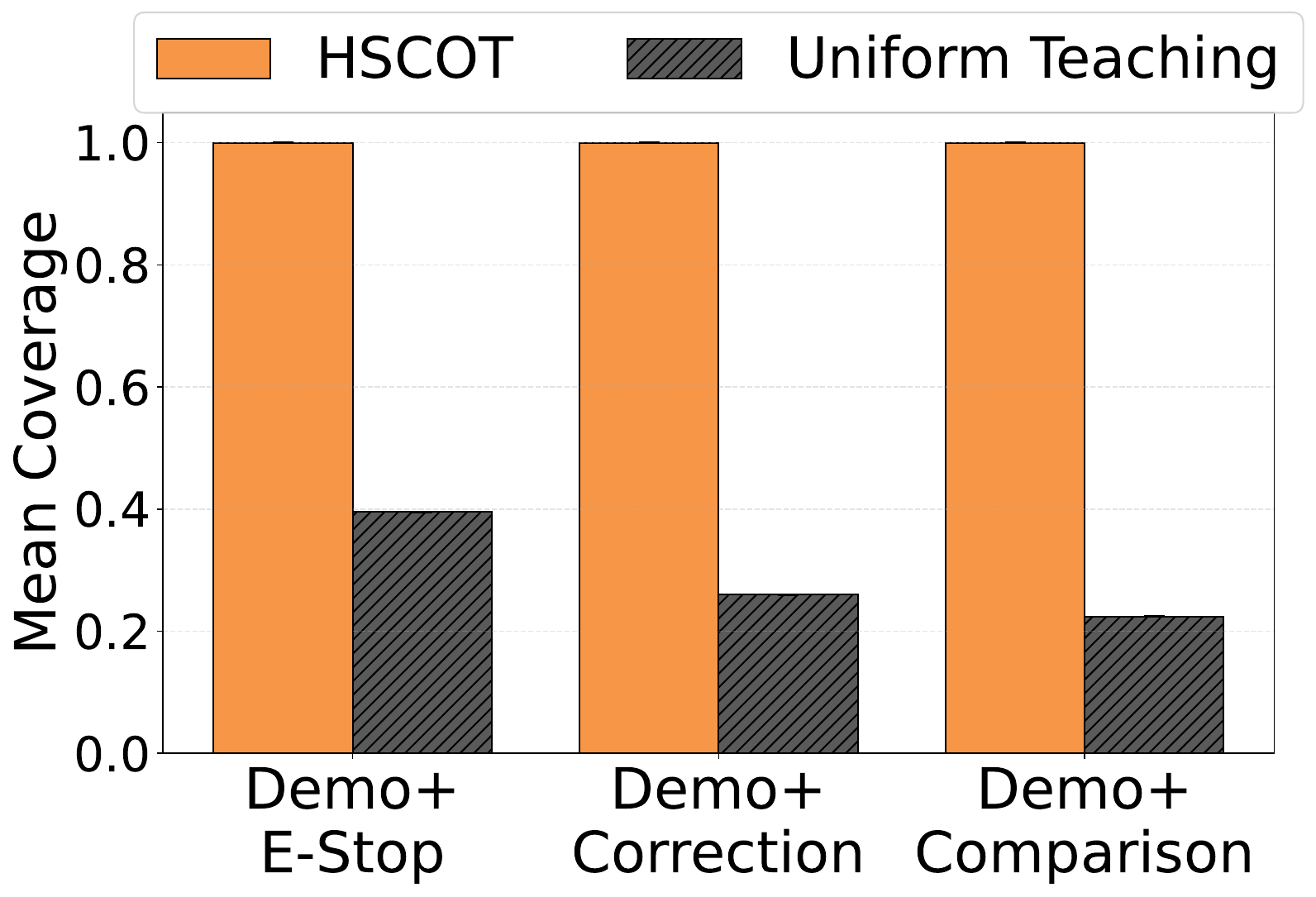}
    \caption{LavaMiniGrid}
    \label{fig:coverage_main_placeholder2}
  \end{subfigure}

  \vspace{4pt}
\caption{\textbf{Constraint coverage} (higher is better) averaged over 10 random seeds. Bars show the fraction of the universal constraint set covered, and error bars denote standard error across seeds. Subfigures (a,b) correspond to single-modality feedback, while (c,d) show mixed-feedback settings where demonstrations are combined with another modality. HSCOT (solid bars) achieves complete coverage in both domains across all settings, whereas uniform teaching (dashed bars) covers only a small fraction of the constraint universe.}
  \label{fig:coverage_main}
\end{figure}
\vspace{-2pt}

\paragraph{Selective environment activation.}
Table~\ref{tab:env_count} reports the average number of distinct environments activated during teaching.
HSCOT consistently uses fewer environments than uniform teaching under the same global budget, reflecting its outer-stage greedy selection of MDPs that contribute maximal marginal constraint coverage.
The results also highlight the necessity of multi-environment teaching under limited feedback.
Each MDP exposes only a subset of the global constraint universe due to its specific transition dynamics, so no single environment suffices to eliminate reward ambiguity.
HSCOT selects a small but informative subset whose joint constraints approximate the full intersection over all training MDPs.
Environment counts are lowest for demonstration feedback, consistent with each demonstration inducing many optimality constraints at once, and highest for E-stop feedback, whose weaker and more trajectory-local constraints require activating more environments to achieve comparable coverage.
Per-environment teaching maps that visualize, for each modality, the LavaMiniGrid layouts, the environments HSCOT activates, the held-out split, and the resulting per-environment regret are provided in Appendix~\ref{app:env_maps}. 

\begin{table}[t]
  \centering
  \caption{Average environments activated (out of 50 training MDPs), mean $\pm$ std over 10 seeds. Aggregate averages the two domain means. Lower is better.}
  \label{tab:env_count}
  \small
  \begin{tabular}{lcccccc}
    \toprule
    Feedback modality & \multicolumn{2}{c}{GridWorld} & \multicolumn{2}{c}{LavaMiniGrid} & \multicolumn{2}{c}{Aggregate} \\
    \cmidrule(lr){2-3} \cmidrule(lr){4-5} \cmidrule(lr){6-7}
                      & HSCOT & Uniform & HSCOT & Uniform & HSCOT & Uniform \\
    \midrule
    Demonstration     & $\mathbf{4.3 \pm 0.4}$ & $5.2 \pm 0.4$ & $\mathbf{5.5 \pm 0.4}$ & $7.1 \pm 0.4$ & \textbf{4.9} & 6.2 \\
    Correction        & $\mathbf{6.9 \pm 0.4}$ & $7.7 \pm 0.5$ & $\mathbf{6.9 \pm 0.6}$ & $7.0 \pm 0.5$ & \textbf{6.9} & 7.4 \\
    Comparison        & $\mathbf{6.3 \pm 0.4}$ & $7.2 \pm 0.4$ & $\mathbf{7.5 \pm 0.3}$ & $7.9 \pm 0.4$ & \textbf{6.9} & 7.6 \\
    E-stop            & $\mathbf{5.8 \pm 0.4}$ & $6.9 \pm 0.4$ & $\mathbf{12.0 \pm 0.4}$ & $15.0 \pm 0.5$ & \textbf{8.9} & 11.0 \\
    \bottomrule
  \end{tabular}
\end{table}

\vspace{-5pt}
\section{Conclusion}
\vspace{-5pt}
We studied how to teach reward functions that remain robust across environment dynamics. Our analysis showed that feedback modalities differ sharply in informativeness: comparisons impose the strongest global constraints with unlimited data, while demonstrations are more constraint-efficient per query under tight budgets. We further proved that reward identifiability is environment-dependent—unlimited feedback in a single MDP can leave residual ambiguity that only additional environments resolve. Building on these insights, we introduced HSCOT, a two-stage greedy algorithm that selects informative environments and then queries feedback within them. Across GridWorld and LavaMiniGrid, HSCOT achieved lower held-out regret, complete constraint coverage, and fewer activated environments than uniform teaching under identical budgets, though extending HSCOT to continuous domains and noisy human feedback remains an important direction.

\bibliographystyle{plainnat}
\bibliography{main}

\beginSupplementaryMaterials
\section{Appendix}
\label{sec:appendix}

\subsection{Proofs}
\label{app:proofs}
\setcounter{theorem}{0}
\setcounter{proposition}{0}
\setcounter{lemma}{0}

\begin{assumption}[Genericity / non-redundancy]
\label{ass:richness}
For each modality $m\in\{\text{demo},\text{corr},\text{E-stop}\}$, there exists a
$w^\star$-consistent comparison ordering
$w^{\star\top}\bigl(\Phi(\xi_i)-\Phi(\xi_j)\bigr)\ge 0$ that is \emph{not} implied by the
constraints of modality $m$; equivalently, $H_m\cap\{w: w^\top(\Phi(\xi_i)-\Phi(\xi_j))<0\}$
has positive measure. This rules out degenerate environments in which the modalities induce
identical constraints, and is what makes the inclusions in
Propositions~\ref{prop:pairwise_vs_demo}--\ref{prop:estop_comparison} strict.
\end{assumption}

\begin{remark}[Sufficient conditions for Assumption~\ref{ass:richness}]
Concretely, Assumption~\ref{ass:richness} holds when:
(i) for Prop.~\ref{prop:pairwise_vs_demo}, there exist suboptimal trajectories
$\xi_a,\xi_b\in\Xi\setminus\Xi^\star$ whose ordering is not implied by the demonstration
(same-initial-state optimality) constraints;
(ii) for Prop.~\ref{prop:comparison_vs_correction}, $M$ admits $\ge 2$ distinct initial states and
some cross-initial-state ordering is not implied by the same-initial-state constraints;
(iii) for Prop.~\ref{prop:estop_comparison}, some $w^\star$-consistent comparison ordering is not
implied by the union of all E-stop (prefix-vs-continuation) constraints across $\Xi$.
\end{remark}

All feasible regions below are intersected with the unit ball $\{w:\|w\|_2\le 1\}$.

\begin{proposition}[Comparisons strictly reduce reward ambiguity compared to demonstrations]
\label{prop:pairwise_vs_demo}

Let $\Xi^\star \subseteq \Xi$ denote the trajectories that are optimal under $w^\star$. 
Demonstrations induce the feasible region
\vspace{-5pt}
\[
H_{\text{demo}} =
\Bigl\{
w \in \mathbb{R}^d : \;
w^\top(\Phi(\xi^\star)-\Phi(\xi')) \ge 0 \;
\forall \xi^\star \in \Xi^\star,\;
\forall \xi' \in \Xi \text{ with the same initial state as } \xi^\star
\Bigr\}.
\]

\vspace{-3pt}
Then
\vspace{-5pt}
\[
H_{\text{comparison}} \subsetneq H_{\text{demo}},
\quad\text{and}\quad
G(H_{\text{comparison}}) < G(H_{\text{demo}}).
\]

\end{proposition}

\begin{proof}
\textbf{Step 1: Inclusion $H_{\text{comparison}} \subseteq H_{\text{demo}}$.}

Demonstrations reveal trajectories that are optimal under the true reward $w^\star$.
Thus for any $\xi^\star \in \Xi^\star$ and any trajectory $\xi'$ with the same initial state,
\[
w^{\star\top}\Phi(\xi^\star) \ge w^{\star\top}\Phi(\xi').
\]

Each demonstration constraint therefore has the form
\[
w^\top(\Phi(\xi^\star)-\Phi(\xi')) \ge 0.
\]

This is a special case of the comparison ordering constraints defining
$H_{\text{comparison}}$, since comparisons compare arbitrary
trajectories $\xi_i,\xi_j \in \Xi$.

Therefore any $w$ satisfying all comparison constraints automatically satisfies
all demonstration constraints, implying

\[
H_{\text{comparison}} \subseteq H_{\text{demo}}.
\]

\textbf{Step 2: Strict inclusion.}

To show the inclusion is strict, it suffices to exhibit a reward vector
$w \in H_{\text{demo}} \setminus H_{\text{comparison}}$.

Consider two suboptimal trajectories
$\xi_a,\xi_b \in \Xi \setminus \Xi^\star$ such that

\[
w^{\star\top}\Phi(\xi_a) > w^{\star\top}\Phi(\xi_b).
\]

Let

\[
v := \Phi(\xi_a) - \Phi(\xi_b).
\]

Demonstration constraints compare optimal trajectories with alternatives
from the same initial state, but do not constrain the relative ordering
between suboptimal trajectories such as $\xi_a$ and $\xi_b$.
Thus the sign of $w^\top v$ is not determined by the demonstration constraints.

Consequently there exists a reward vector
$w \in H_{\text{demo}}$ such that

\[
w^\top v < 0.
\]

Such a vector satisfies all demonstration constraints but violates the
comparison ordering $\xi_a \succ \xi_b$, implying
$w \notin H_{\text{comparison}}$.

Therefore

\[
H_{\text{comparison}} \subsetneq H_{\text{demo}}.
\]

\textbf{Step 3: Volume comparison.}

Since $H_{\text{comparison}}$ is obtained from $H_{\text{demo}}$
by adding additional linear inequality constraints,
and both sets are compact convex subsets of the unit sphere
(after $\|w\|_2=1$ normalization),
their volumes satisfy

\[
G(H_{\text{comparison}}) < G(H_{\text{demo}}).
\]

This completes the proof.
\end{proof}

\begin{proposition}[Comparisons strictly reduce reward ambiguity compared to corrective feedback]
\label{prop:comparison_vs_correction}

Corrective feedback induces the feasible region
\vspace{-5pt}
\[
H_{\text{correction}} =
\Bigl\{
w \in \mathbb{R}^d : \;
w^\top(\Phi(\xi_{\mathrm{corr}})-\Phi(\xi_R)) \ge 0 \;
\forall \xi_{\mathrm{corr}}, \xi_R \in \Xi
\text{ with the same initial state}
\Bigr\}.
\]
\vspace{-3pt}
Then
\vspace{-5pt}
\[
H_{\text{comparison}} \subsetneq H_{\text{correction}},
\quad\text{and}\quad
G(H_{\text{comparison}}) < G(H_{\text{correction}}).
\]

\end{proposition}

\begin{proof}

\textbf{Step 1: Inclusion $H_{\text{comparison}} \subseteq H_{\text{correction}}$.}

Each corrective feedback instance compares two trajectories
$\xi_{\mathrm{corr}}$ and $\xi_R$ that share the same initial state,
imposing

\[
w^\top(\Phi(\xi_{\mathrm{corr}})-\Phi(\xi_R)) \ge 0.
\]

This is a special case of a comparison constraint.
Hence every correction constraint appears among the comparison constraints,
implying

\[
H_{\text{comparison}} \subseteq H_{\text{correction}}.
\]

\textbf{Step 2: Strict inclusion.}

Consider trajectories $\xi_a,\xi_b \in \Xi$ with different initial states
such that

\[
w^{\star\top}\Phi(\xi_a) > w^{\star\top}\Phi(\xi_b).
\]

The corresponding comparison constraint is

\[
w^\top(\Phi(\xi_a)-\Phi(\xi_b)) \ge 0.
\]

However, correction constraints only compare trajectories that start
from the same state. Therefore the ordering between $\xi_a$ and $\xi_b$
is not enforced by correction constraints.

Thus there exists $w \in H_{\text{correction}}$ such that

\[
w^\top(\Phi(\xi_a)-\Phi(\xi_b)) < 0.
\]

Hence $w \notin H_{\text{comparison}}$ and

\[
H_{\text{comparison}} \subsetneq H_{\text{correction}}.
\]

\textbf{Step 3: Volume comparison.}

Since $H_{\text{comparison}}$ adds additional linear constraints to
$H_{\text{correction}}$, their volumes satisfy

\[
G(H_{\text{comparison}}) < G(H_{\text{correction}}).
\]

\end{proof}





\begin{lemma}[E-stop constraints are trajectory-local]
\label{lemma:distincttraj}
Fix a trajectory $\xi \in \Xi$.
Any E-stop constraint comparing $\xi^{0:t}$ with $\xi$ lies in the subspace
\[
\operatorname{span}
\{\phi(s,a) \mid (s,a) \text{ appears in } \xi\}.
\]
In contrast, a pairwise comparison between $\xi$ and any $\xi' \in \Xi$ 
that visits a state-action pair $(s',a')$ with $\phi(s',a')$ linearly 
independent of this subspace produces a feature difference that lies 
outside it.
\end{lemma}

\begin{proof}

Let

\[
\xi = (s_0,a_0,s_1,a_1,\dots,s_T,a_T)
\]

be a trajectory.

For any stopping time $t<T$, the halted trajectory is $\xi^{0:t}$.

The discounted feature sums are

\[
\Phi(\xi)
=
\sum_{\tau=0}^{T} \gamma^\tau \phi(s_\tau,a_\tau)
\]

and

\[
\Phi(\xi^{0:t})
=
\sum_{\tau=0}^{t} \gamma^\tau \phi(s_\tau,a_\tau)
+
\sum_{\tau=t+1}^{\infty} \gamma^\tau \phi(s_t,a_t).
\]

The difference vector is

\[
\Phi(\xi^{0:t})-\Phi(\xi)
=
\frac{\gamma^{t+1}}{1-\gamma}\phi(s_t,a_t)
-
\sum_{\tau=t+1}^{T}\gamma^\tau\phi(s_\tau,a_\tau).
\]

All terms involve only feature vectors $\phi(s_\tau,a_\tau)$
for state-action pairs appearing in $\xi$.

Therefore

\[
\Phi(\xi^{0:t})-\Phi(\xi)
\in
\operatorname{span}
\{\phi(s,a)\mid (s,a)\text{ appears in }\xi\}.
\]

Now consider another trajectory $\xi'$ that visits a state-action pair
$(s',a')$ whose feature vector is linearly independent of this span.
Then the comparison difference

\[
\Phi(\xi)-\Phi(\xi')
\]

contains a component outside the trajectory-specific subspace,
which cannot be generated by E-stop constraints along $\xi$ alone.

\end{proof}

\begin{proposition}[Comparisons strictly reduce reward ambiguity compared to E-stop feedback]
\label{prop:estop_comparison}

E-stop feedback induces the feasible region
\vspace{-5pt}
\[
H_{\text{E-stop}} =
\Bigl\{
w \in \mathbb{R}^d : \|w\|_2 \le 1,\;
w^\top(\Phi(\xi^{0:t})-\Phi(\xi)) \ge 0 \;
\forall \xi \in \Xi,\;
\forall t
\Bigr\}.
\]
\vspace{-3pt}
Then
\vspace{-5pt}
\[
H_{\text{comparison}} \subsetneq H_{\text{E-stop}},
\quad\text{and}\quad
G(H_{\text{comparison}}) < G(H_{\text{E-stop}}).
\]

\end{proposition}

\begin{proof}

\textbf{Step 1: Inclusion.}

Each E-stop constraint compares two trajectories $\xi^{0:t}$ and $\xi$.
This is a special case of a comparison
Thus every E-stop constraint appears among the comparison constraints,
implying

\[
H_{\text{comparison}} \subseteq H_{\text{E-stop}}.
\]

\textbf{Step 2: Strict inclusion.}

By Lemma~\ref{lemma:distincttraj},
E-stop constraints along a trajectory $\xi$ only constrain reward
vectors in the subspace spanned by features appearing in $\xi$.

Comparisons between trajectories $\xi$ and $\xi'$ can introduce
feature differences outside this subspace.

Hence there exist reward vectors satisfying all E-stop constraints
but violating a comparison ordering.

Therefore

\[
H_{\text{comparison}} \subsetneq H_{\text{E-stop}}.
\]

\textbf{Step 3: Volume comparison.}

Since $H_{\text{comparison}}$ adds additional constraints to $H_{\text{E-stop}}$,
their volumes satisfy

\[
G(H_{\text{comparison}}) < G(H_{\text{E-stop}}).
\]

\end{proof}





\begin{theorem}[Single-MDP ambiguity]
\label{thm:single_mdp_ambiguity}
Let $\mathcal{M}=\{M_1,M_2\}$ be two finite MDPs that share a common 
feature map $\phi:\mathcal{S}\times\mathcal{A}\rightarrow\mathbb{R}^d$ 
and a common ground-truth reward parameter $w^\star \in \mathbb{R}^d$, 
where rewards take the linear form $r(s,a)=w^{\star\top}\phi(s,a)$.
For any MDP $M_k$, let $\mathcal{W}(M_k)$ denote the generalized behavioral
equivalence class (gBEC), i.e.,
the set of reward parameters consistent with \emph{all possible}
idealized feedback constraints obtainable within $M_k$.
Define the feature-difference span induced by $M_k$ as
\[
\mathcal{V}_k
=\operatorname{span}\{\Delta\Phi(\xi^+,\xi^-)\;:\;
\xi^+,\xi^- \text{ are feasible trajectories in } M_k\},
\]
where $\Delta\Phi(\xi^+,\xi^-)=\Phi(\xi^+)-\Phi(\xi^-)$
and $\Phi(\xi)$ denotes the feature expectation of trajectory $\xi$.
Assume that the induced spans satisfy 
$\mathcal{V}_1 \subsetneq \operatorname{span}(\mathcal{V}_1\cup\mathcal{V}_2)$. 
Then there exists a reward parameter $w\neq w^\star$ such that 
$w \in \mathcal{W}(M_1)$ but $w \notin \mathcal{W}(M_1)\cap\mathcal{W}(M_2)$. 
\end{theorem}

\begin{proof}
\textbf{Step 1: Cone form of the gBEC.}
For each MDP $M_k$, let
\[
\mathcal{C}_k=\bigl\{\Delta\Phi(\xi^+,\xi^-) : \xi^+,\xi^- \text{ feasible in } M_k,\
(w^\star)^\top\Delta\Phi(\xi^+,\xi^-)\ge 0\bigr\}
\]
be the feature-difference directions that form valid constraints under $w^\star$.
Then
\[
\mathcal{W}(M_k)=\{w\in\mathbb{R}^d : w^\top\Delta\Phi\ge 0\ \ \forall\,\Delta\Phi\in\mathcal{C}_k\}.
\]
For any pair $(\xi_i,\xi_j)$, one of $\Delta\Phi(\xi_i,\xi_j)$ or its negation
$\Delta\Phi(\xi_j,\xi_i)$ satisfies the sign condition, so $\mathcal{C}_k$ contains a
spanning vector for every line generated by a feature difference, giving
$\operatorname{span}(\mathcal{C}_k)=\mathcal{V}_k$.

\textbf{Step 2: Perturbation direction $u$.}
Since $\mathcal{V}_1\subsetneq\operatorname{span}(\mathcal{V}_1\cup\mathcal{V}_2)$, we
have $\mathcal{V}_2\not\subseteq\mathcal{V}_1$. Fix $z\in\mathcal{V}_2\setminus\mathcal{V}_1$.
Let $\{b_1,\dots,b_m\}$ be an orthonormal basis of $\mathcal{V}_1$ and define
\[
u=z-\sum_{i=1}^{m}(z^\top b_i)\,b_i,
\]
the component of $z$ orthogonal to $\mathcal{V}_1$. Then $u\neq 0$ (as
$z\notin\mathcal{V}_1$), and $u^\top\Delta\Phi=0$ for all $\Delta\Phi\in\mathcal{V}_1$.
Moreover $u^\top z=\|u\|_2^2>0$ with $z\in\mathcal{V}_2$, so $u\not\perp\mathcal{V}_2$.

\textbf{Step 3: A violable constraint.}
Since $\operatorname{span}(\mathcal{C}_2)=\mathcal{V}_2$ and $u\not\perp\mathcal{V}_2$,
there exists $\Delta\Phi'\in\mathcal{C}_2$ with $u^\top\Delta\Phi'\neq 0$. Replacing
$z$ by $-z$ if its projection gives the wrong sign, take $u$ so that
\[
u^\top\Delta\Phi'<0,\qquad (w^\star)^\top\Delta\Phi'\ge 0,
\]
where the second inequality holds because $\Delta\Phi'\in\mathcal{C}_2$.

\textbf{Step 4: Explicit perturbation magnitude.}
Choose
\[
\varepsilon>\frac{(w^\star)^\top\Delta\Phi'}{\lvert u^\top\Delta\Phi'\rvert}\ \ (\ge 0).
\]
Then
\[
(w^\star+\varepsilon u)^\top\Delta\Phi'
=(w^\star)^\top\Delta\Phi'-\varepsilon\,\lvert u^\top\Delta\Phi'\rvert<0,
\]
for any margin $(w^\star)^\top\Delta\Phi'\ge 0$.

\textbf{Step 5: Witness reward.}
Set $\tilde w=w^\star+\varepsilon u\neq w^\star$. For every
$\Delta\Phi\in\mathcal{C}_1$, $u^\top\Delta\Phi=0$, so
\[
\tilde w^\top\Delta\Phi=(w^\star)^\top\Delta\Phi\ge 0,
\]
giving $\tilde w\in\mathcal{W}(M_1)$. By Step 4, $\tilde w^\top\Delta\Phi'<0$, so
$\tilde w\notin\mathcal{W}(M_2)$ and hence
$\tilde w\notin\mathcal{W}(M_1)\cap\mathcal{W}(M_2)$. Therefore no feedback obtained
solely from $M_1$ identifies $w^\star$.
\end{proof}





\subsection{Per-budget feasible region analysis}
\label{app:budget_sweep}

To analyze the limited-budget regime more rigorously, we go beyond the single
fixed-budget heatmaps and study how each feedback modality behaves as the budget
varies. We proceed in two steps. First, we visualize the feasible reward region
as a heatmap across a range of feedback budgets $B\in\{1,2,3,5,7,10,15,20\}$,
showing how the region shrinks for each modality as more queries become
available. Second, to obtain a quantitative, budget-matched comparison, we
approximate the \emph{volume} of the feasible region at each budget and plot it
as a function of $B$. All figures in this section are computed for the same
single MDP in Figure~\ref{fig:grid1}, matching the finite-data regions in
Figure~\ref{fig:limite_feasible_region}.

\paragraph{Per-budget heatmap results.}
The per-budget feasibility heatmaps (Figures~\ref{fig:heatmaps_budget}
and~\ref{fig:heatmaps_budget_cont}) show how each modality contracts the feasible
region as the budget grows. Demonstrations are tight from the start: already at
$B=1$ (Fig.~\ref{fig:hb1}) the demonstration region is a narrow wedge around
$w^\star$ and it changes little thereafter, since a single demonstration
contrasts the optimal action at each visited state against its alternatives and
thus instantiates many optimality constraints at once. Comparisons narrow
steadily: the comparison region starts as a wide wedge at $B=1$ but contracts at
each successive budget, becoming a thin sliver comparable to demonstrations by
$B=15$--$20$ (Fig.~\ref{fig:hb15}--\ref{fig:hb20}), because comparisons can
contrast arbitrary trajectory pairs and each query may cut along a new
feature-difference direction. Corrections and E-stops, by contrast, saturate:
they reduce the feasible region far more slowly and retain wide wedges even at
$B=20$, since by Lemma~\ref{lemma:distincttraj} E-stop constraints lie in the
span of features along a single trajectory and corrections are limited by their
shared start-state structure, so repeated queries re-cut along similar directions
and quickly stop adding information. Taken together, the heatmaps make a regime
crossover visible: demonstrations dominate under tight budgets, while comparisons
overtake them as the budget grows. This trend is consistent with our
unlimited-data analysis (Figure~\ref{fig:feasible_region}): as the budget
increases, the comparison wedge approaches and begins to match the demonstration
wedge, moving toward the unlimited-data regime in which comparisons impose the
strongest global ordering constraints and yield the smallest feasible region
(Proposition~\ref{prop:pairwise_vs_demo}). The crossover is only partial at the
budgets shown---corrections and E-stops remain far from their idealized regions
because their constraints saturate---but the comparison trend clearly anticipates
the unlimited-data ordering.

\paragraph{Feasible-region volume results.}
Aggregating these regions into a single scalar, Figure~\ref{fig:volume_vs_budget}
plots the feasible-region volume $G(D)$ against the budget $B$, providing the
budget-matched comparison across modalities. Demonstrations attain the smallest
volume at every budget, collapsing the feasible region to a thin sliver already
at $B=1$ and leaving little room to shrink further. Comparisons begin with a much
larger volume but contract steadily, approaching the demonstration curve only at
large budgets ($B\ge 15$), consistent with comparisons needing many queries
to accumulate the global ordering constraints that make them dominant in the
unlimited-data regime. Corrections and E-stops reduce ambiguity far more slowly
and plateau at large volumes---E-stop the largest---because their constraints are
confined to same-start-state or trajectory-local feature-difference subspaces
(\ Lemma~\ref{lemma:distincttraj}) and saturate quickly. The ordering
$G(D)_{\text{demo}}\le G(D)_{\text{comp}}\le G(D)_{\text{corr}}\le
G(D)_{\text{E-stop}}$ holds across the entire budget range with tight variance
bands, confirming that demonstrations are the most constraint-efficient modality
under limited budgets. This steady contraction of the comparison curve toward the
demonstration curve mirrors the crossover anticipated by the unlimited-data
analysis (Figure~\ref{fig:feasible_region}): with enough queries, comparisons
accumulate the global ordering constraints that make them the most constraining
modality in the limit, so increasing the budget moves the empirical ordering
toward the idealized one in which comparisons yield the smallest feasible region.
The convergence is clearest for comparisons; corrections and E-stops plateau well
above their idealized regions, indicating that the budgets considered here
approach but do not fully reach the unlimited-data regime for those modalities.

\paragraph{How the feasible-region volume is approximated.}
For each budget $B$ and each modality, we draw $B$ random feedback queries and
extract the induced homogeneous linear constraints $v^\top w \ge 0$, where
$v=\Phi(\xi^+)-\Phi(\xi^-)$ is the discounted feature-difference between the
preferred and dispreferred trajectory. The feasible reward region is the set of
weights satisfying all of these constraints simultaneously. Because every
constraint is homogeneous and we restrict to $\|w\|_2\le 1$, this region is a
convex cone intersected with the unit disk, i.e., an angular wedge. We
approximate its volume $G(D)$ as the \emph{fraction of the unit reward disk that
remains feasible}: we evaluate all constraints on a dense grid over the disk and report the proportion of points satisfying every constraint. Smaller $G(D)$ thus indicates a tighter, less ambiguous reward region. We repeat the procedure $100$ times per $(B,\text{modality})$ configuration and report the mean and standard deviation, yielding the curves in Figure~\ref{fig:volume_vs_budget}.

\begin{figure}[t]
  \centering
  \includegraphics[width=0.62\textwidth]{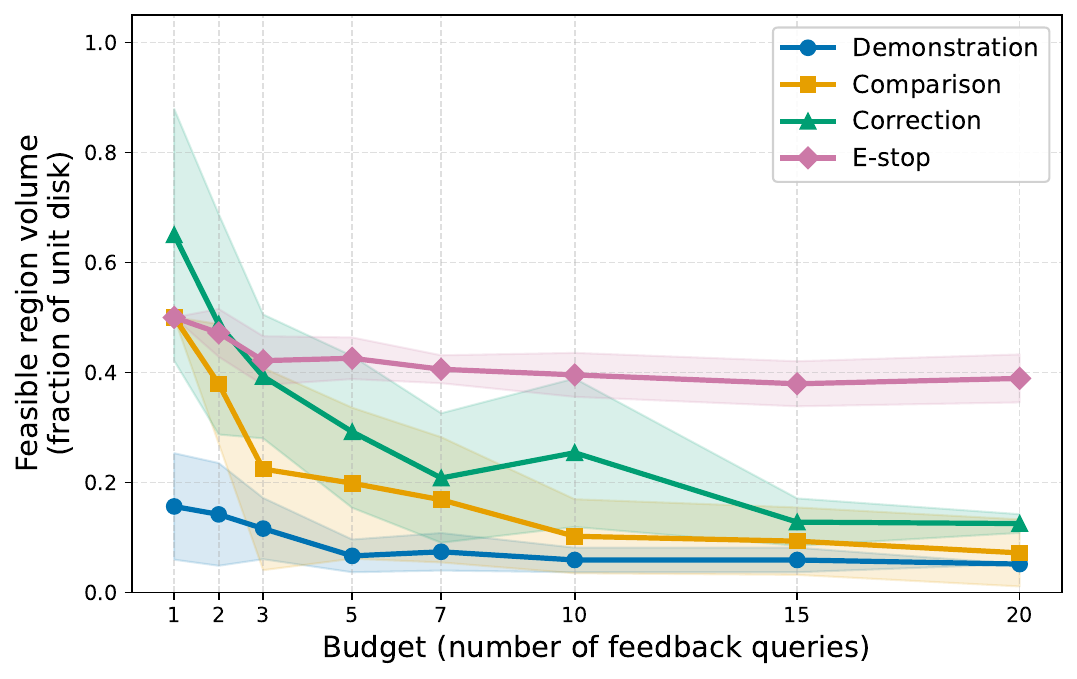}
  \caption{\textbf{Budget-matched reward ambiguity across modalities.}
  Feasible-region volume $G(D)$ (lower is better) versus budget $B$, averaged over
  $100$ random samples per setting; shaded bands denote $\pm 1$ standard deviation.
  Demonstrations attain the smallest volume at every budget; comparisons shrink
  steadily and approach demonstrations only for large $B$; corrections and E-stops
  reduce ambiguity slowly and plateau at larger volumes. The ordering
  $G(D)_{\text{demo}}\le G(D)_{\text{comp}}\le G(D)_{\text{corr}}\le
  G(D)_{\text{E-stop}}$ holds across the entire range with tight variance bands.}
  \label{fig:volume_vs_budget}
\end{figure}

\begin{figure}[p]
  \centering
  \begin{subfigure}{\textwidth}
    \includegraphics[width=\linewidth]{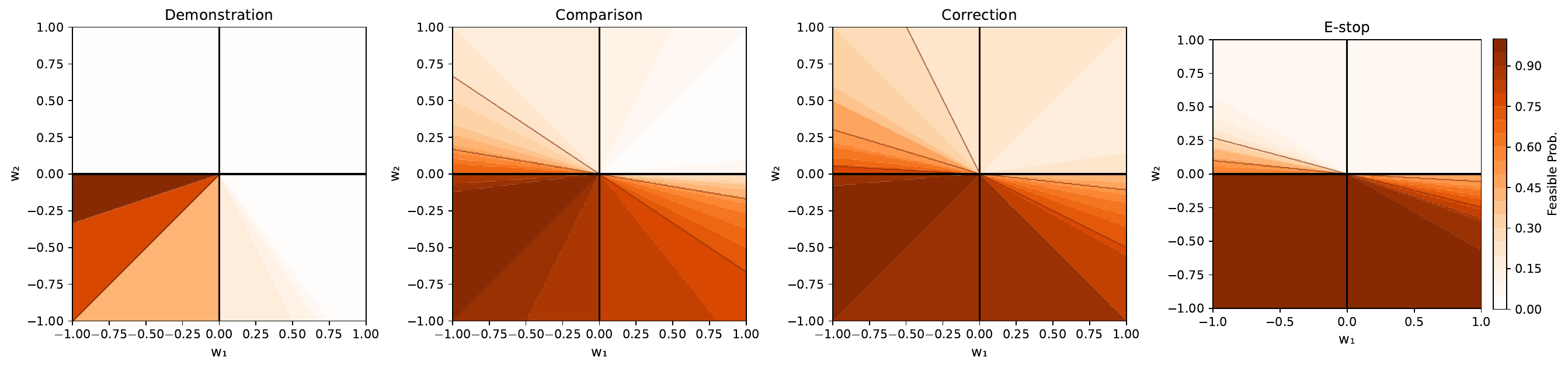}
    \caption{$B=1$}
    \label{fig:hb1}
  \end{subfigure}\\[3pt]
  \begin{subfigure}{\textwidth}
    \includegraphics[width=\linewidth]{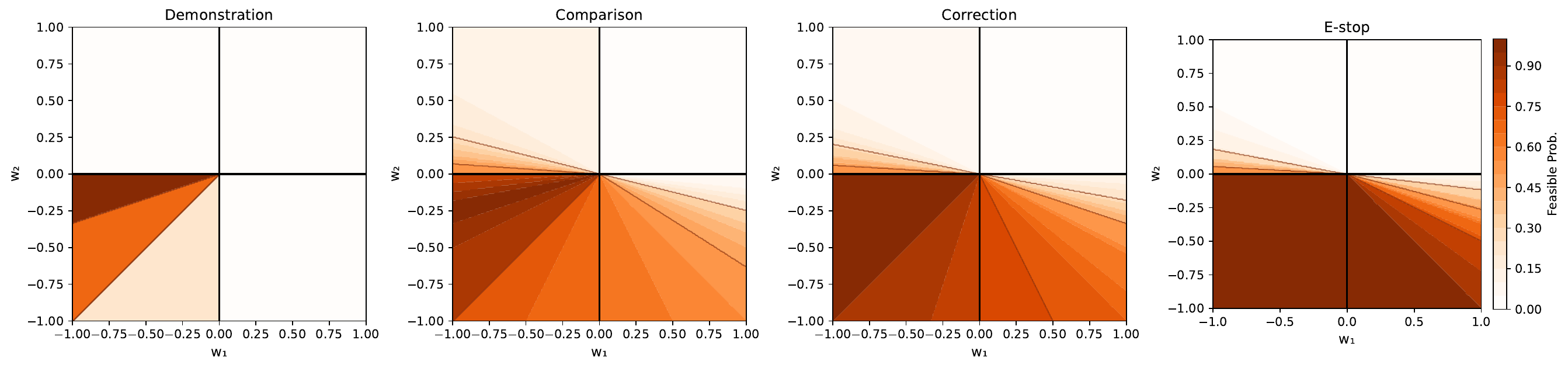}
    \caption{$B=2$}
    \label{fig:hb2}
  \end{subfigure}\\[3pt]
  \begin{subfigure}{\textwidth}
    \includegraphics[width=\linewidth]{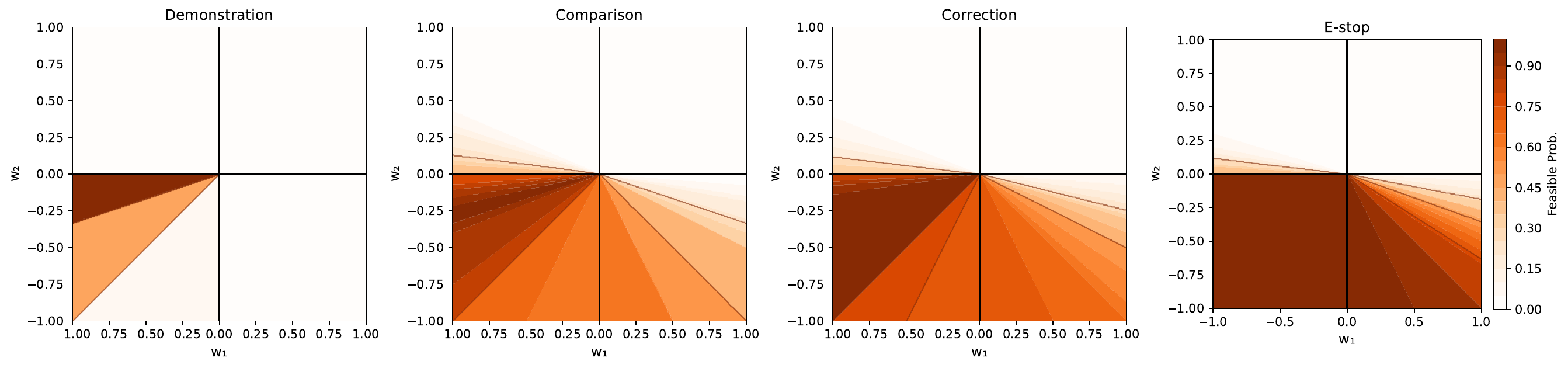}
    \caption{$B=3$}
    \label{fig:hb3}
  \end{subfigure}\\[3pt]
  \begin{subfigure}{\textwidth}
    \includegraphics[width=\linewidth]{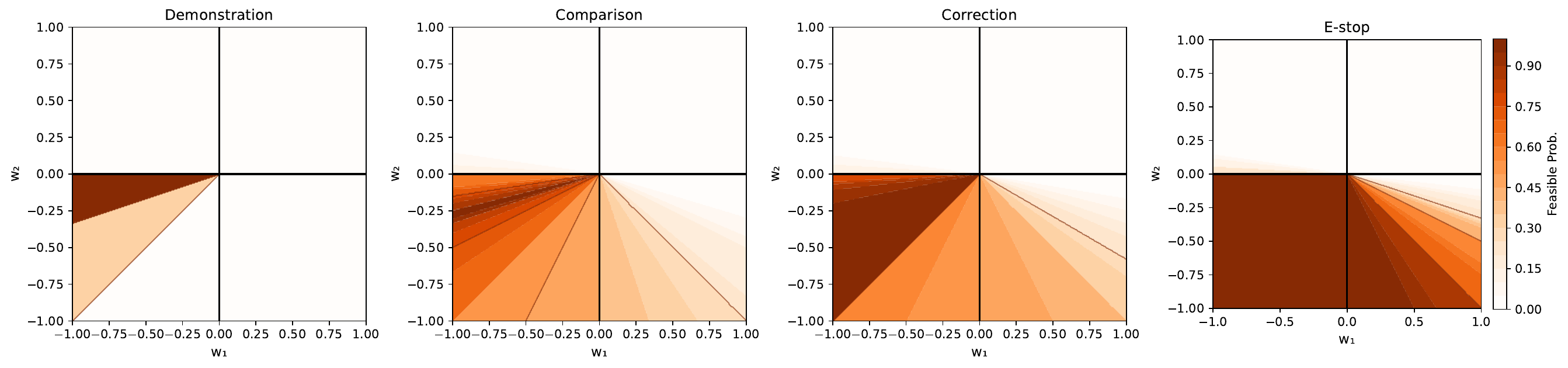}
    \caption{$B=5$}
    \label{fig:hb5}
  \end{subfigure}
  \caption{\textbf{Per-budget feasible reward regions across feedback modalities.}
  Each row shows the feasible reward region for a single feedback budget $B$, with
  the four columns corresponding to demonstration, comparison, correction, and
  E-stop feedback. Colors encode the empirical feasibility probability over $100$
  independent random feedback samples---darker regions are feasible across more
  samples. All panels are
  computed for the same MDP whose layout is shown in Figure~\ref{fig:grid1}.
  \emph{(Continued in Figure~\ref{fig:heatmaps_budget_cont}.)}}
  \label{fig:heatmaps_budget}
\end{figure}

\begin{figure}[p]\ContinuedFloat
  \centering
  \begin{subfigure}{\textwidth}
    \includegraphics[width=\linewidth]{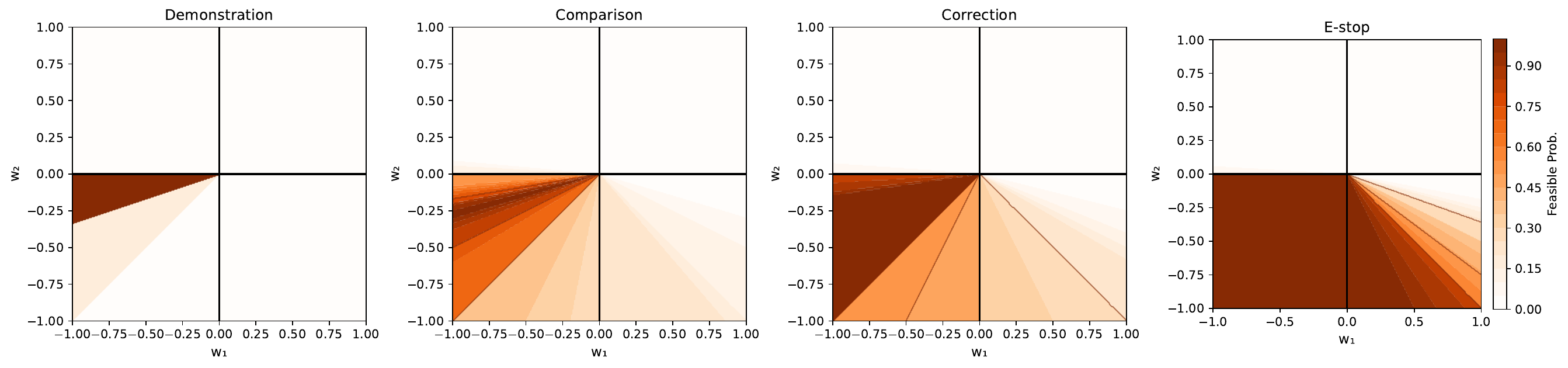}
    \caption{$B=7$}
    \label{fig:hb7}
  \end{subfigure}\\[3pt]
  \begin{subfigure}{\textwidth}
    \includegraphics[width=\linewidth]{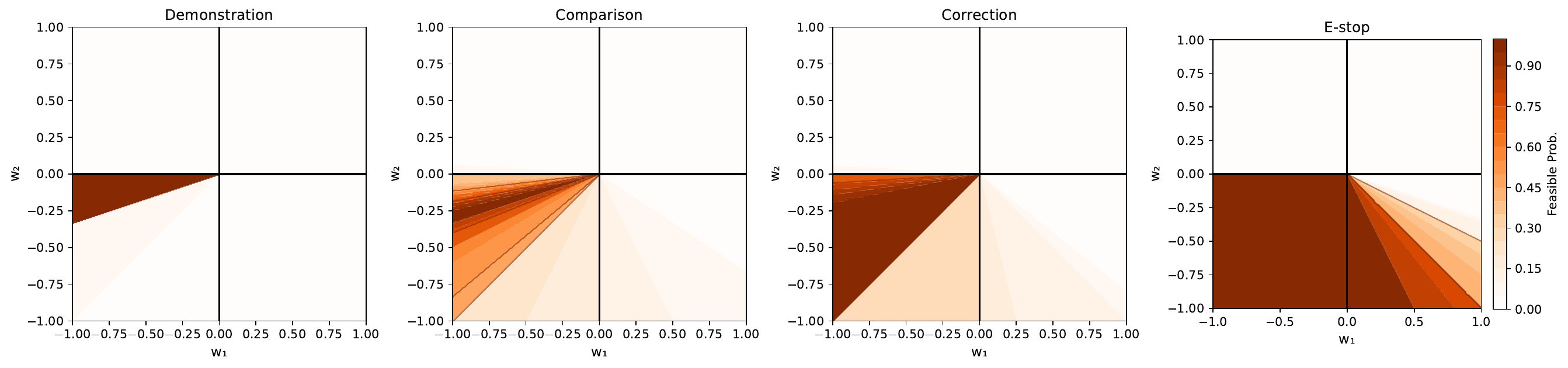}
    \caption{$B=10$}
    \label{fig:hb10}
  \end{subfigure}\\[3pt]
  \begin{subfigure}{\textwidth}
    \includegraphics[width=\linewidth]{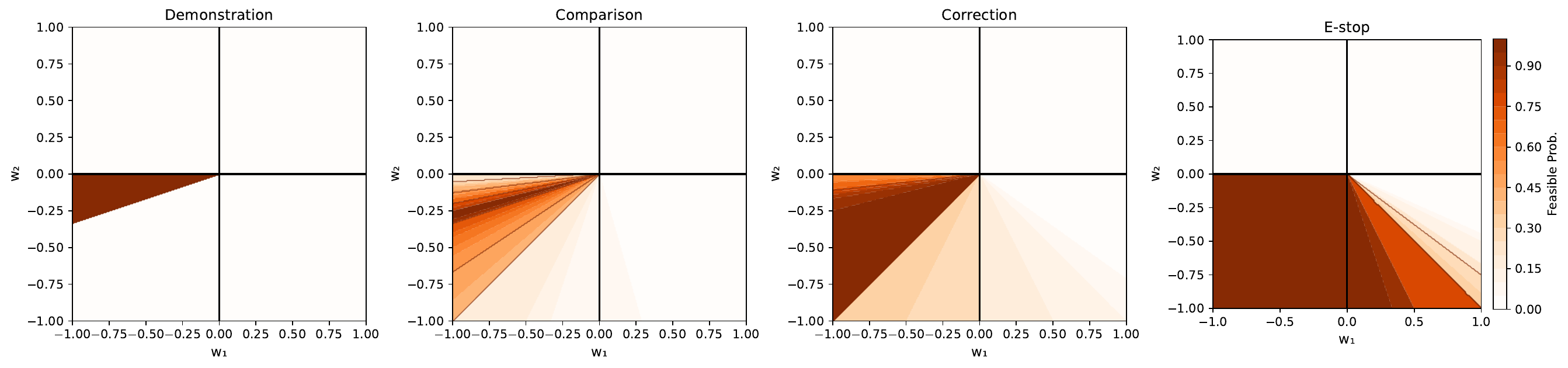}
    \caption{$B=15$}
    \label{fig:hb15}
  \end{subfigure}\\[3pt]
  \begin{subfigure}{\textwidth}
    \includegraphics[width=\linewidth]{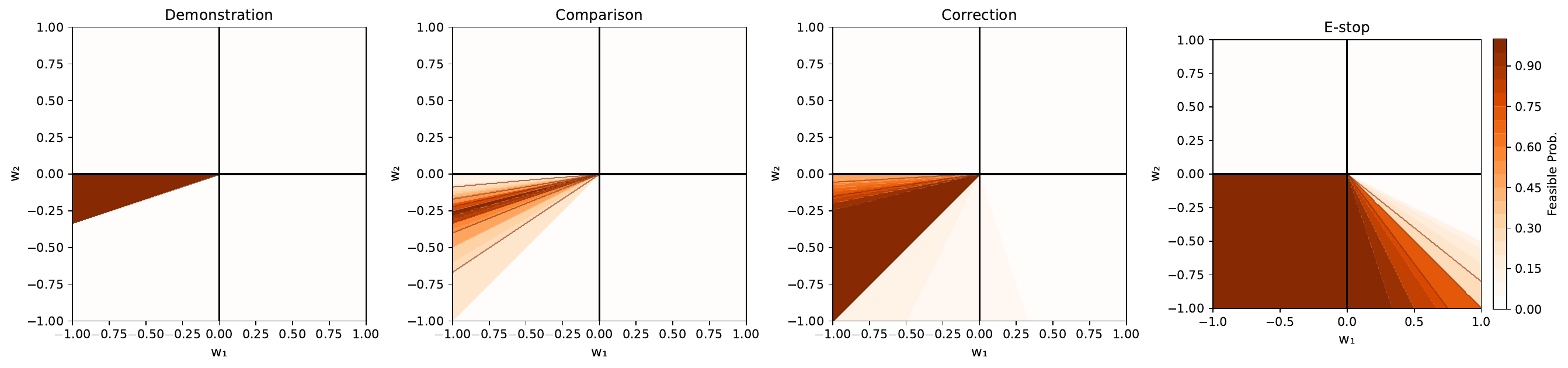}
    \caption{$B=20$}
    \label{fig:hb20}
  \end{subfigure}
  \caption{\textbf{Per-budget feasible reward regions (continued).} Budgets
  $B\in\{7,10,15,20\}$, continuing Figure~\ref{fig:heatmaps_budget}. As the budget
  grows, the comparison wedge tightens toward the demonstration wedge, while the
  correction and E-stop wedges shrink slowly and stay comparatively wide,
  reflecting their same-start-state and trajectory-local constraint structure.}
  \label{fig:heatmaps_budget_cont}
\end{figure}

\subsection{Feasible reward regions across sampled layouts}
\label{app:layout_gallery}

Figures~\ref{fig:layout_gallery} and~\ref{fig:feasible_gallery} extend the single MDP-pair
illustration of Figure~\ref{fig:reward_ambiguity_asymmetry} to a larger population of layouts,
confirming that the environment-dependent ambiguity of Theorem~\ref{thm:single_mdp_ambiguity} is
typical rather than an artifact of one hand-picked example. We use $2\times3$ gridworlds with two
cell features (drawn gray and white) and a randomly placed terminal cell $\mathbf{T}$ that may
occupy either feature. For each $k\in\{1,\dots,5\}$, where $k$ is the number of gray-feature cells,
we sample $40$ layouts, for $200$ layouts in total. Figure~\ref{fig:layout_gallery} shows the
sampled layouts with their optimal policies (one arrow per tied optimal action), and
Figure~\ref{fig:feasible_gallery} shows the corresponding feasible reward regions
$\mathrm{gBEC}(D)$. The feasible region varies markedly across layouts: some pin the reward down to
a thin wedge, while others---typically those with more gray cells, whose optimal trajectories are
less distinguishable in feature space---leave a much larger region. This reflects that reward
identifiability is governed by the layout-induced feature-difference directions rather than by the
amount of feedback.

\paragraph{How the feasible regions are computed.}
Each region is obtained from the constraints induced by \emph{all} optimal demonstrations in a
layout, expressed through action-successor features. Recall the discounted feature counts
(successor features) of a policy $\pi$ introduced in the preliminaries, which we write here as
$\psi^\pi(s)\equiv\mu^\pi(s)$ and $\psi^\pi(s,a)\equiv\mu^\pi(s,a)$:

\[
\psi^\pi(s)=\mathbb{E}_\pi\!\Big[\textstyle\sum_{t=0}^\infty \gamma^t \phi(s_t,a_t)\ \big|\ s_0=s\Big],
\]
\[
\psi^\pi(s,a)=\mathbb{E}\!\Big[\textstyle\sum_{t=0}^\infty \gamma^t \phi(s_t,a_t)\ \big|\ s_0=s,\,a_0=a,\ a_{t\ge1}\sim\pi\Big].
\]

The action-successor feature $\psi^\pi(s,a)$ takes action $a$ in state $s$ and then follows $\pi$; it
captures the full long-run feature consequence of that action choice and satisfies
$Q^\pi_w(s,a)=w^\top\psi^\pi(s,a)$ (in these gridworlds the feature is a property of the cell, so
$\phi(s,a)=\phi(s)$). We obtain the optimal policy $\pi^\star$ by value iteration and compute its
action-successor features $\psi^{\pi^\star}$. Optimality of any $a^\star\in\arg\max_a Q^{\pi^\star}_w(s,a)$
over an alternative action $b$ then gives, for every non-terminal state $s$,
\[
w^\top\big(\psi^{\pi^\star}(s,a^\star)-\psi^{\pi^\star}(s,b)\big)\ge 0,
\]
which is exactly a behavioral-equivalence-class optimality constraint and a special case of the
generalized constraint $w^\top\Delta\Phi\ge0$ in Eq.~\eqref{eq:gbec_def}, with
$\Delta\Phi=\psi^{\pi^\star}(s,a^\star)-\psi^{\pi^\star}(s,b)$. Collecting these feature-difference
vectors over all states and all optimal/non-optimal action pairs, and discarding constraints implied
by the rest, yields the half-spaces whose intersection is $\mathrm{gBEC}(D)$. Restricting to
$\|w\|_2\le1$ makes this a bounded convex cone (an angular wedge for $d=2$): in
Figure~\ref{fig:feasible_gallery} the red lines are the retained constraint boundaries and the gold
hatch is their feasible intersection.

\begin{figure}[t]
  \centering
  \includegraphics[width=\linewidth]{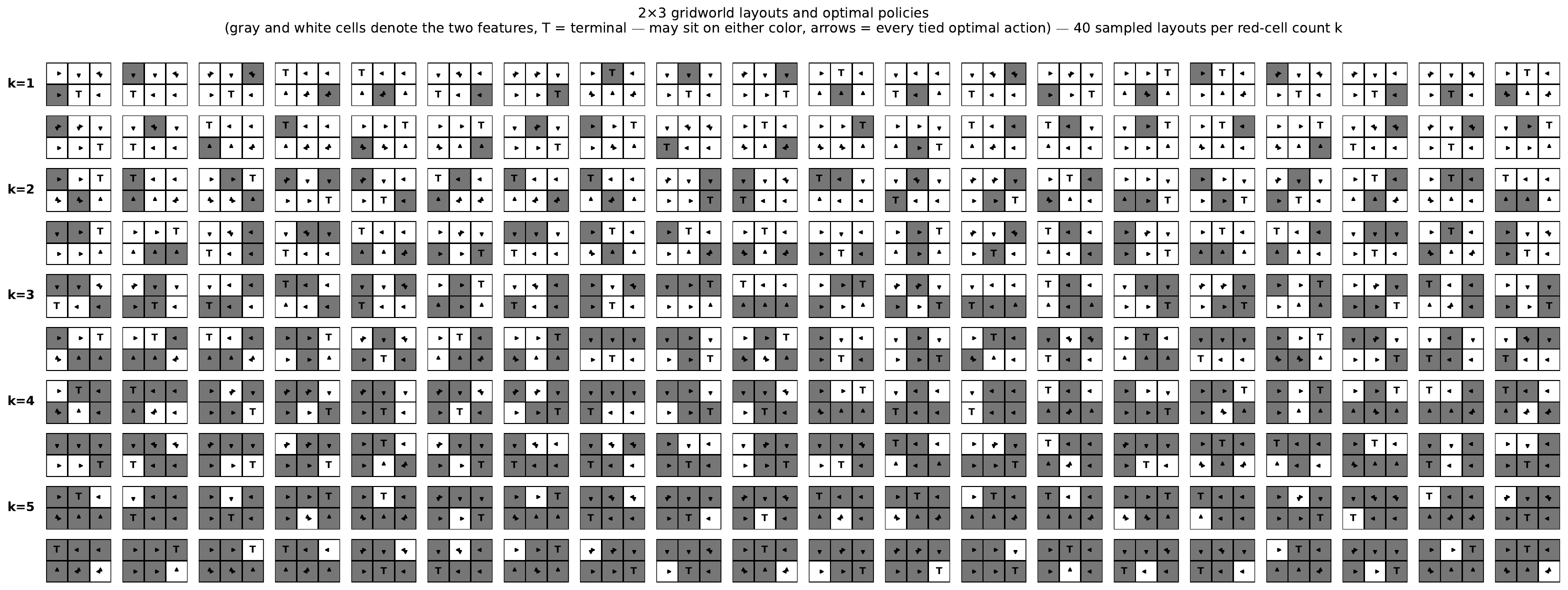}
  \caption{\textbf{Sampled $2\times3$ layouts and optimal policies.} Gray and white cells denote the
  two features and $\mathbf{T}$ the (randomly placed) terminal cell, which may sit on either feature;
  arrows show every tied optimal action. Rows are grouped by $k$, the number of gray-feature cells
  ($40$ layouts per $k$, $200$ total).}
  \label{fig:layout_gallery}
\end{figure}

\begin{figure}[t]
  \centering
  \includegraphics[width=\linewidth]{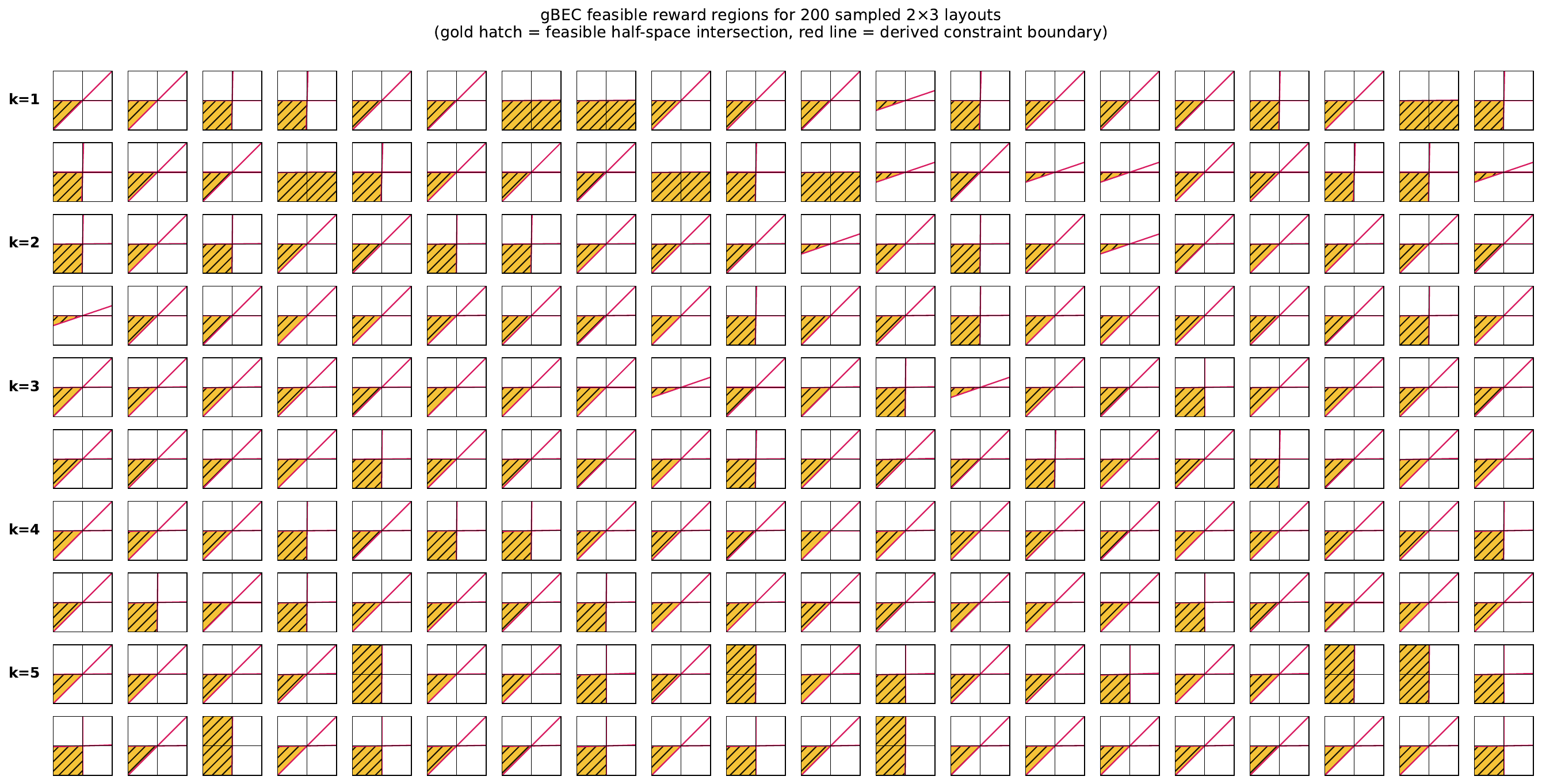}
  \caption{\textbf{Feasible reward regions $\mathrm{gBEC}(D)$ for the sampled layouts.} For each
  layout in Figure~\ref{fig:layout_gallery}, the gold hatch is the feasible half-space intersection
  and the red lines are the derived constraint boundaries, computed from the action-successor-feature
  constraints described above. The region shape---and hence reward identifiability---varies strongly
  with layout, mirroring Figure~\ref{fig:reward_ambiguity_asymmetry}.}
  \label{fig:feasible_gallery}
\end{figure}

\subsection{Supplementary Algorithms}
\label{sec:supp_algorithms}

This section provides the full pseudocode for the hierarchical
set-cover teaching procedure used in the main text.
The algorithms correspond to the two-stage greedy approximation
described in Section~\ref{sec:hierarchical_teaching}.

\subsection*{Algorithm S1: Greedy Environment Selection (Outer Stage)}

\begin{algorithm}[H]
\caption{Greedy Environment Selection}
\label{alg:env_selection}
\begin{algorithmic}[1]
\Require Per-MDP constraint coverage sets $\{\mathcal{U}_k\}_{k \in \mathcal{M}_{\mathrm{train}}}$, universe $\mathcal{U}$
\Ensure Ordered list of selected MDP indices $\mathcal{K}$

\State $\text{covered} \gets \emptyset$
\State $\mathcal{K} \gets [\,]$

\While{$\text{covered} \neq \mathcal{U}$}
    \State Select $k$ maximizing $|\mathcal{U}_k \setminus \text{covered}|$
    \State $\mathcal{K} \gets \mathcal{K} \cup \{k\}$
    \State $\text{covered} \gets \text{covered} \cup (\mathcal{U}_k \setminus \text{covered})$
\EndWhile

\State \Return $\mathcal{K}$
\end{algorithmic}
\end{algorithm}

\vspace{5mm}

\subsection*{Algorithm S2: Greedy Atom Selection (Inner Stage)}

\begin{algorithm}[H]
\caption{Greedy Atom Selection}
\label{alg:atom_selection}
\begin{algorithmic}[1]
\Require Selected MDPs $\mathcal{K}$, candidate atoms and coverage sets within $\mathcal{K}$, universe $\mathcal{U}$
\Ensure Ordered list of chosen atoms $D$

\State $\text{covered} \gets \emptyset$
\State $D \gets [\,]$

\While{$\text{covered} \neq \mathcal{U}$}
    \State Select atom $x$ from any $k \in \mathcal{K}$ maximizing $|\mathrm{coverage}(x) \setminus \text{covered}|$
    \State $D \gets D \cup \{x\}$
    \State $\text{covered} \gets \text{covered} \cup (\mathrm{coverage}(x) \setminus \text{covered})$
\EndWhile

\State \Return $D$
\end{algorithmic}
\end{algorithm}

\vspace{5mm}

\subsection*{Algorithm S3: Hierarchical SCOT (HSCOT)}

\begin{algorithm}[H]
\caption{Hierarchical SCOT (HSCOT)}
\label{alg:hscot}
\begin{algorithmic}[1]
\Require Constraint universe $\mathcal{U}$, per-MDP inducible constraints $\{\mathcal{U}_k\}$, candidate atoms per MDP
\Ensure Selected environments $\mathcal{K}$ and feedback dataset $D$

\State $\mathcal{K} \gets$ Greedy Environment Selection$(\{\mathcal{U}_k\}, \mathcal{U})$ \hfill (Alg. S1)

\State Restrict candidate atoms to those in environments $\mathcal{K}$

\State $D \gets$ Greedy Atom Selection$(\mathcal{K}, \mathcal{U})$ \hfill (Alg. S2)

\State \Return $(\mathcal{K}, D)$
\end{algorithmic}
\end{algorithm}

\subsection{Per-environment teaching maps}
\label{app:env_maps}

To complement the aggregate metrics in Section~\ref{sec:experiments}, we
visualize the full set of $50$ LavaMiniGrid environments for each feedback
modality. In every panel, red cells denote lava, green the goal, and white free
space. An orange border marks environments selected by HSCOT, a cyan border marks
held-out environments excluded from teaching, and a yellow border marks training
environments with high residual regret. Each panel title reports its
per-environment regret $r$, and each figure title reports the mean regret over
all environments.
These maps make HSCOT's selective environment activation concrete. Comparison
(Figure~\ref{fig:envmap_comparison}), correction
(Figure~\ref{fig:envmap_correction}), and demonstration
(Figure~\ref{fig:envmap_demo}) all achieve zero mean regret, whereas E-stop
(Figure~\ref{fig:envmap_estop}) leaves residual regret across many environments
(mean regret $0.188$).

\begin{figure}[p]
  \centering
  \begin{subfigure}{\textwidth}
    \includegraphics[width=\linewidth]{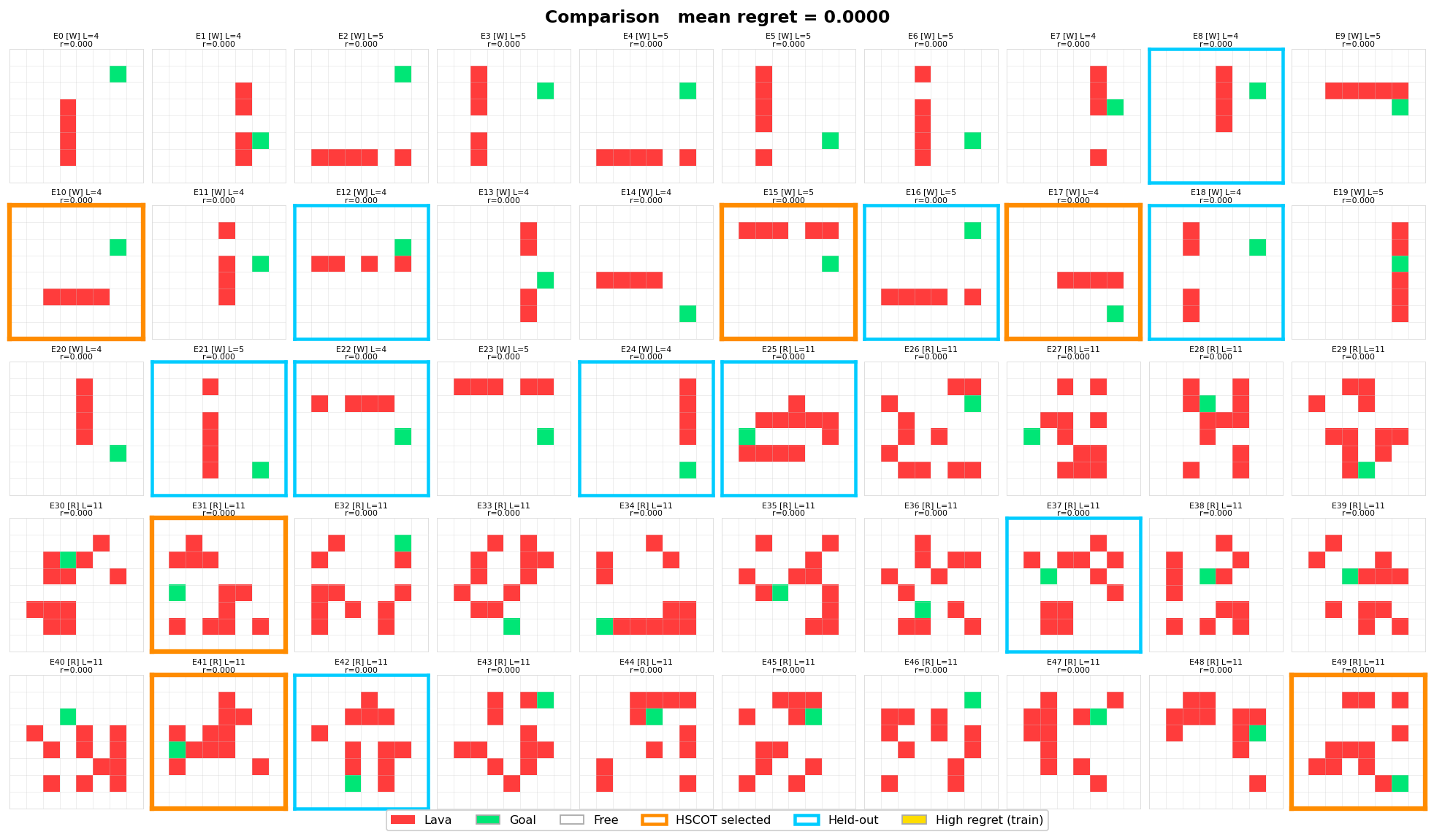}
    \caption{Comparison feedback (mean regret $=0.000$).}
    \label{fig:envmap_comparison}
  \end{subfigure}\\[4pt]
  \begin{subfigure}{\textwidth}
    \includegraphics[width=\linewidth]{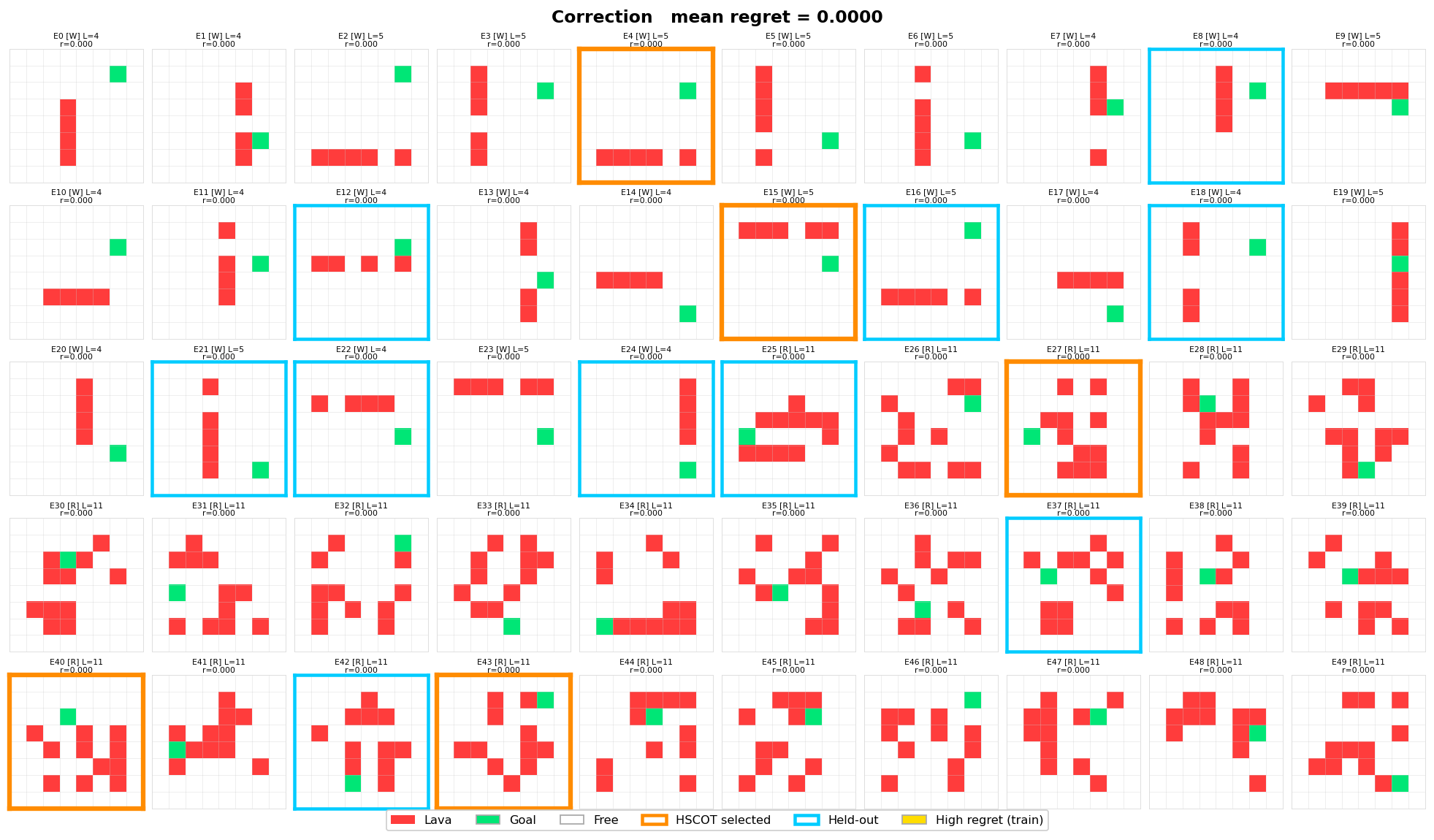}
    \caption{Correction feedback (mean regret $=0.000$).}
    \label{fig:envmap_correction}
  \end{subfigure}
  \caption{\textbf{Per-environment teaching maps (comparison and correction).}
  Each panel shows one LavaMiniGrid environment: red is lava, green the goal, white
  free space. Orange borders mark HSCOT-selected environments, cyan borders mark
  held-out environments, and yellow borders mark high-regret training environments.
  Panel titles report per-environment regret $r$.}
  \label{fig:env_maps_part1}
\end{figure}

\begin{figure}[p]
  \centering
  \begin{subfigure}{\textwidth}
    \includegraphics[width=\linewidth]{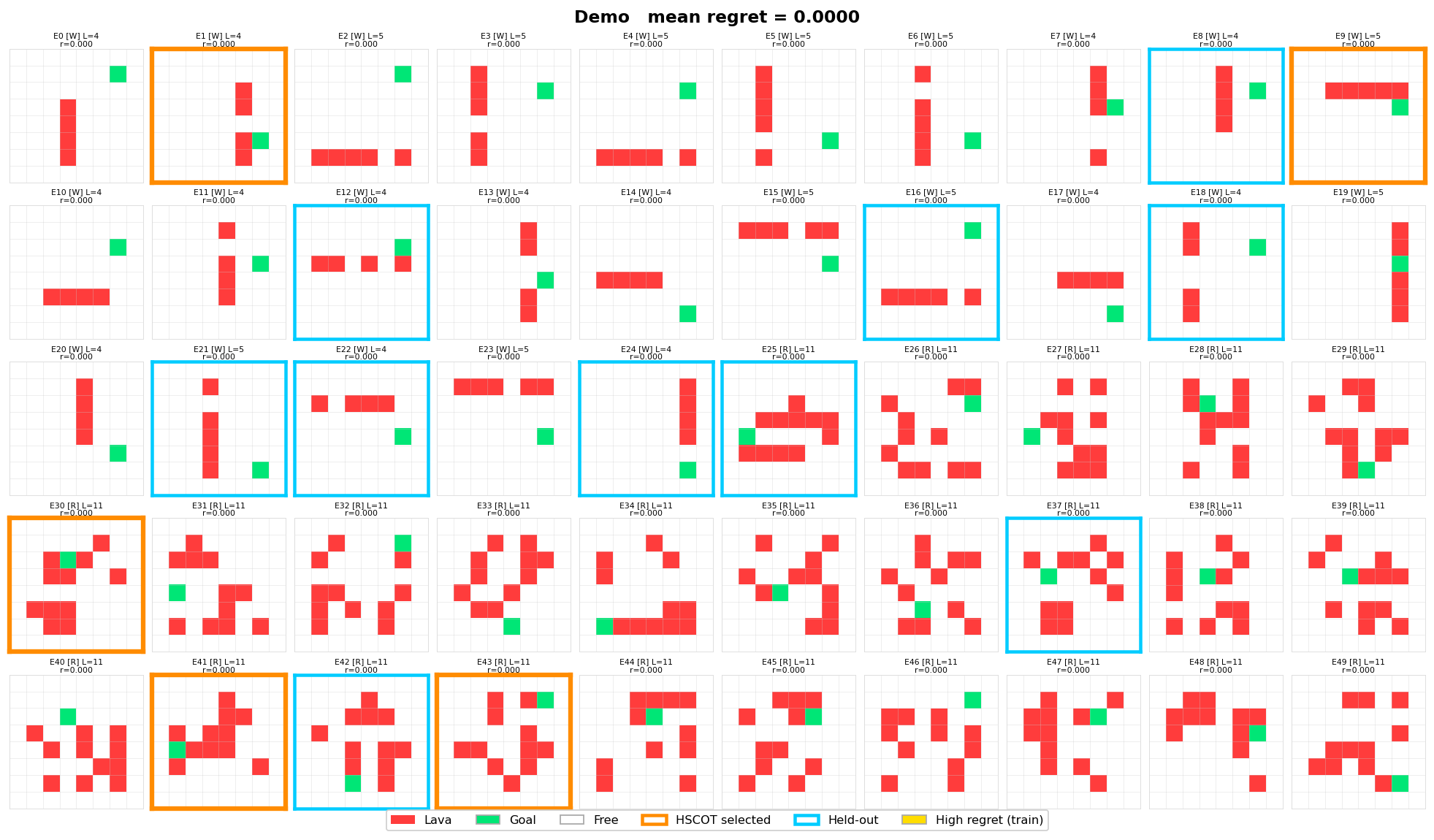}
    \caption{Demonstration feedback (mean regret $=0.000$).}
    \label{fig:envmap_demo}
  \end{subfigure}\\[4pt]
  \begin{subfigure}{\textwidth}
    \includegraphics[width=\linewidth]{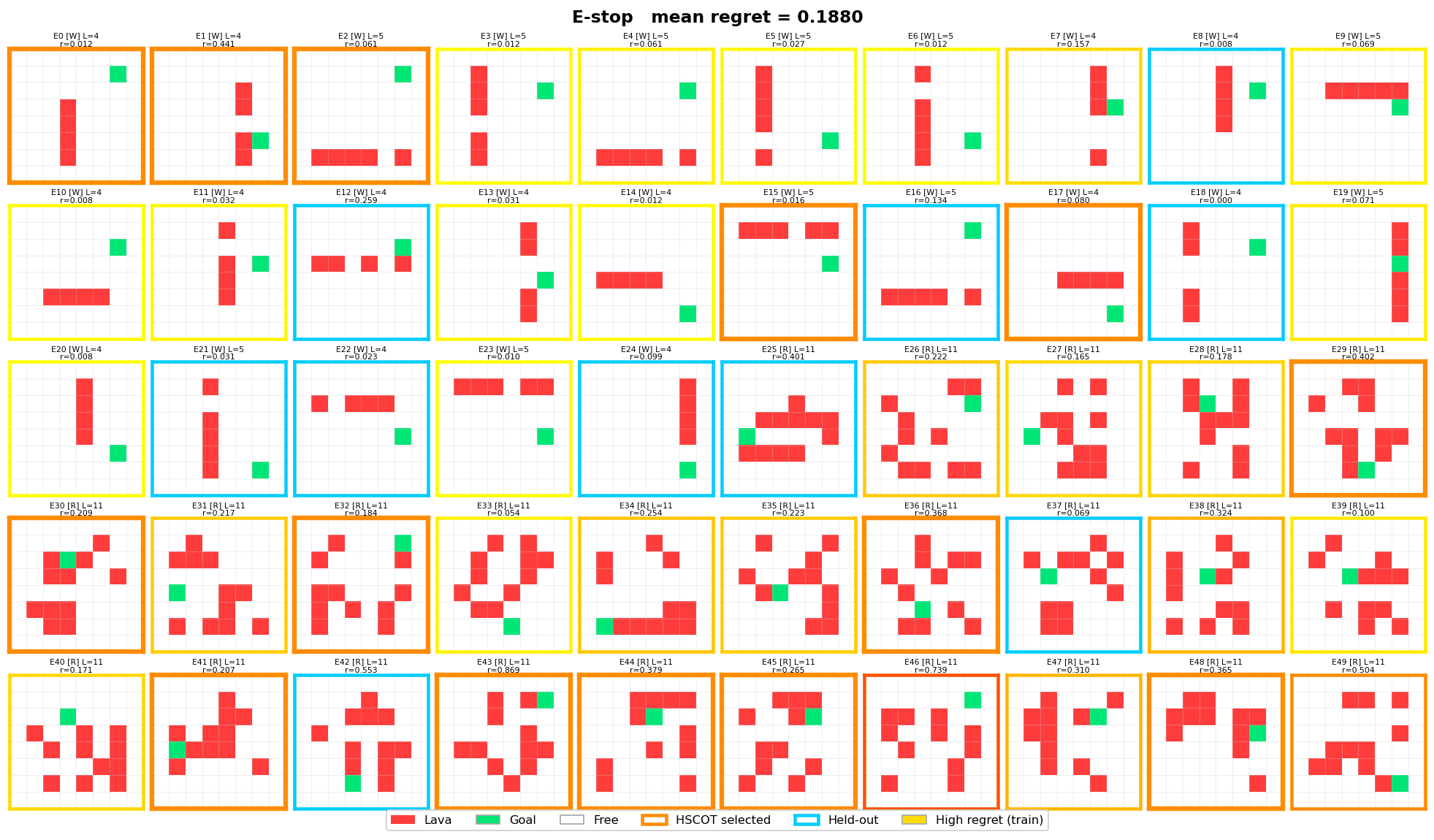}
    \caption{E-stop feedback (mean regret $=0.188$).}
    \label{fig:envmap_estop}
  \end{subfigure}
  \caption{\textbf{Per-environment teaching maps (demonstration and E-stop).}
  Same color and border conventions as Figure~\ref{fig:env_maps_part1}. E-stop
  leaves substantial residual regret across many environments and triggers several
  high-regret (yellow) panels, unlike the other modalities.}
  \label{fig:env_maps_part2}
\end{figure}

\end{document}